\documentclass[10pt]{article}
\usepackage[letterpaper, margin=1in]{geometry}
\usepackage{amsmath, amsthm}
\usepackage{newtx}
\usepackage{graphicx}
\usepackage[textsize=scriptsize, color=lightgray, backgroundcolor=lightgray]{todonotes}
\usepackage{hyperref}
\usepackage{cleveref}
\usepackage{bbm}
\usepackage{xcolor}
\usepackage{thmtools,thm-restate}
\usepackage{algorithm, algpseudocode}
\usepackage[square,numbers]{natbib}
\usepackage{authblk}
\usepackage{tikz}
\usetikzlibrary{positioning,arrows.meta}

\newtheorem{assumption}{Assumption}

\newtheorem{theorem}{Theorem}[section]

\newtheorem{lemma}[theorem]{Lemma}
\newtheorem{definition}[theorem]{Definition}
\newtheorem{proposition}[theorem]{Proposition}
\newtheorem{corollary}[theorem]{Corollary}
\newtheorem{remark}[theorem]{Remark}

\newcommand{\eps}{\varepsilon}
\renewcommand{\epsilon}{\varepsilon}
\renewcommand{\hat}{\widehat}
\renewcommand{\tilde}{\widetilde}
\renewcommand{\bar}{\overline}

\newcommand{\abs}[1]{\left|#1\right|}

\DeclareMathOperator*{\Exp}{\mathbb{E}}

\newcommand{\EEs}[2]{\Exp_{#1}\left[#2\right]}

\newcommand{\EEsc}[3]{\Exp_{#1}\left[#2 \mid #3\right]}

\renewcommand{\cite}[1]{\citep{#1}}

\DeclareMathOperator*{\argmin}{arg\,min}

\newcommand{\cA}{\mathcal{A}}
\newcommand{\cB}{\mathcal{B}}
\newcommand{\cC}{\mathcal{C}}
\newcommand{\cD}{\mathcal{D}}

\newcommand{\cF}{\mathcal{F}}
\newcommand{\cG}{\mathcal{G}}
\newcommand{\cH}{\mathcal{H}}

\newcommand{\cL}{\mathcal{L}}

\newcommand{\cO}{\mathcal{O}}
\newcommand{\cP}{\mathcal{P}}

\newcommand{\cR}{\mathcal{R}}

\newcommand{\cX}{\mathcal{X}}

\newcommand{\cY}{\mathcal{Y}}

\newcommand{\bp}{\boldsymbol{p}}

\newcommand{\bh}{{\mathbf{h}}}

\begin{document}

\title{Panprediction: Optimal Predictions for \\ Any Downstream Task and Loss}
\author[1]{Sivaraman Balakrishnan}
\author[2]{Nika Haghtalab}
\author[3]{Daniel Hsu}
\author[2]{Brian Lee}
\author[2]{Eric Zhao}

\affil[1]{Carnegie Mellon University}
\affil[2]{University of California, Berkeley}
\affil[3]{Columbia University}
\date{}

\maketitle

\begin{abstract}
Supervised learning is classically formulated as training a model to minimize a fixed loss function over a fixed distribution, or task. 
However, an emerging paradigm instead views model training as extracting enough information from data so that the model can be used to minimize many losses on many downstream tasks. 
We formalize a mathematical framework for this paradigm, which we call \textit{panprediction}, and study its statistical complexity. Formally, panprediction generalizes omniprediction and sits upstream from multi-group learning, which respectively focus on predictions that generalize to many downstream losses or many downstream tasks, but not both. 
Concretely, we design algorithms that learn deterministic and randomized panpredictors with $\tilde{O}(1/\varepsilon^3)$ and $\tilde{O}(1/\varepsilon^2)$ samples, respectively.
Our results demonstrate that under mild assumptions, simultaneously minimizing infinitely many losses on infinitely many tasks can be as statistically easy as minimizing one loss on one task. 
Along the way, we improve the best known sample complexity guarantee of deterministic omniprediction by a factor of $1/\eps$, and match all other known sample complexity guarantees of omniprediction and multi-group learning. 
Our key technical ingredient is a nearly lossless reduction from panprediction to a statistically efficient notion of calibration, called \textit{step calibration}. 
\end{abstract}

\section{Introduction}
\label{sec:intro}
Consider the problem of predicting the probability of an adverse medical event for a patient based on their health records---be it in-hospital mortality over 12 hours, re-admission over 30 days, or a cardiovascular event over 10 years. 
A range of decision makers across the healthcare system---ICU doctors, discharge planners, actuaries, and more---wish to incorporate these probabilities into their decisions. 
However, this problem is complicated by the fact that each decision maker defines prediction quality differently: an ICU doctor who must rapidly take action may care most about the zero-one loss, while an actuary who seeks precise probability estimates may care most about the square loss. 
Further complicating the problem is how each decision maker focuses on different, possibly overlapping subgroups of patients: those with specific pre-existing conditions, those in a certain age group, and more. 

\begin{quote}
    \textit{What does it mean for a single predictor to be ``good" for such a heterogeneous population of decision makers? 
    How can such a predictor be learned? }
\end{quote}

\begin{figure*}[t]
\centering %
\begin{tikzpicture}[
  node distance=1cm and 2cm,
  every node/.style={align=center},
  box/.style={
    draw,
    minimum width=1cm,      
    minimum height=0.75cm,   
    rounded corners=3pt,
    font=\small,
    inner sep=6pt           
  },
  note/.style={
    font=\small,
    align=center
  },
  arrow/.style={
    ->,
    >=Stealth,
    thin
  }
]

  \node[box] (step) {$\bigl(\cG,\cH,\eps\bigr)$-Step Calibration};
  \node[note, above=0.1cm of step]
    {
      \textbf{Deterministic} $\tilde{O}\left( \eps^{-3}\right)$ \\
      \textbf{Randomized} $\tilde{O}\left( \eps^{-2} \right)$ 
    };

  \node[box, below=0.3cm of step] (combo)
    {$\bigl(\cL, \cG,\cH,\eps\bigr)$-Panprediction};
  \node[note, below=0.05cm of combo]
    {
      \textbf{Deterministic} $\tilde{O}\left( \eps^{-3} \right)$ \\
      \textbf{Randomized} $\tilde{O}\left( \eps^{-2} \right)$ 
    };

  \node[box, left=1cm of combo] (calibleft)
    {$(\ell,\cG,\cH,\eps)$-Multi-group Learning};
  \node[note, below=0.2cm of calibleft]
    {
      Match best known det. and \\ rand. sample complexity \cite{DBLP:conf/icml/Tosh022}
    };

  \node[box, right=1cm of combo] (calibright)
    {$\bigl(\cL,\cH,\eps\bigr)$-Omniprediction};
  \node[note, below=0.2cm of calibright]
    {
      Improve det. sample complexity by $\eps^{-1}$, \\ match rand. sample complexity \cite{okoroafor_near-optimal_2025}
    };

  \draw[arrow] (step)   -- (combo);
  \draw[arrow] (combo)  -- (calibleft);
  \draw[arrow] (combo)  -- (calibright);

\end{tikzpicture}
\caption{Panprediction is a nearly lossless generalization of omniprediction and multi-group learning in the sense that sample complexity guarantees are preserved (up to logarithmic factors), or improved. 
Similarly, the reduction from panprediction to step calibration preserves sample complexity guarantees.}
\label{fig:flow}
\end{figure*}
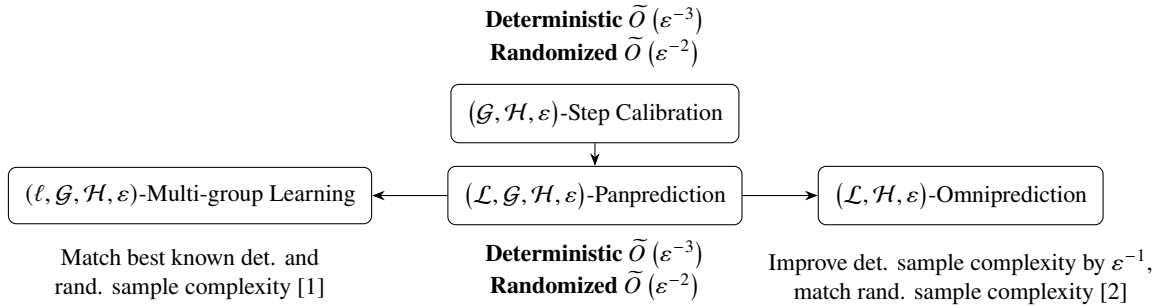

From one perspective, we have simply described a collection of binary prediction problems, each defined by a loss function and subgroup, or task, of interest. 
Classic supervised learning techniques can readily produce a good model for each problem. 
However, this approach requires that the number of samples needed to produce good predictions for all decision makers scales linearly in the number of losses and subgroups of interest, which can be prohibitively large.

In this work, we introduce \textit{panprediction}, an alternative framework for constructing a single probability predictor that any decision maker---focused on any loss and any subgroup---can easily post-process to solve their specialized problem.
Importantly, panpredictors guarantee that decision makers could not have made better predictions by training a bespoke model for their problem. 
Moreover, the sample complexity of panprediction scales only logarithmically in the number of \textit{basis functions} needed to linearly approximate losses of interest, and also in the number of subgroups.

Formally, we say a probability predictor $p^\star: \mathcal{X} \to [0,1]$ is a $(\mathcal{L}, \mathcal{G}, \mathcal{H}, \varepsilon)$-panpredictor if for every loss $\ell \in \mathcal{L}$ and group $g \in \cG$, the predictor $p^\star$ can be post-processed into a predictor $h^\star$ that is $\eps$-optimal, meaning

\begin{align*}
    &\mathbb{E}_{(x,y)\sim D}\bigl[\ell\bigl( h^*(x), y\bigr) \bigm| g(x)=1\bigr] \approx_\eps \min_{h\in\mathcal{H}} \mathbb{E}_{(x,y)\sim D}\bigl[\ell\bigl(h(x), y\bigr) \bigm| g(x)=1\bigr],
\end{align*}
where $\mathcal{H}$ is a competitor hypothesis class. 
Supposing for simplicity that $\cG$ and $\cH$ have finite cardinality, our proposed algorithms learn deterministic panpredictors with $\tilde{O}(\eps^{-3}\log(|\cG||\cH|))$ samples, and randomized panpredictors with $\tilde{O}(\eps^{-2}\log(|\cG||\cH|))$ samples. 
Notably, our sample complexity guarantees account for $\cL$ through a $\log(1/\varepsilon)$ factor, with no dependence on the cardinality or covering number of $\cL$. 

Our results are enabled by the observation that panprediction can be reduced to \textit{step calibration} \cite{qiao_truthfulness_2025}, a statistically efficient notion of calibration also known as proper calibration \cite{okoroafor_near-optimal_2025}. 
This answers our first question: a single predictor is ``good" for a heterogeneous array of losses and subgroups if it is step calibrated. 
Algorithmically, we construct step calibrated predictors using tools from multi-objective learning, a powerful framework recently developed to tighten the sample complexity of various calibration notions \cite{DBLP:conf/nips/HaghtalabJ023, zhang2023optimal}. 
This answers our second question on how such ``good" predictors can be learned. 

Panprediction builds on recent advances in machine learning theory that seek models with greater adaptability to many downstream losses or tasks. 
Omniprediction \cite{gopalan_omnipredictors_2022} studies predictors whose outputs can be post-processed to perform well according to many loss functions, but for a single fixed distribution, or task. 
Conversely, multi-group learning \cite{DBLP:conf/icml/RothblumY21,DBLP:conf/icml/Tosh022} ensures robust performance across many downstream subgroups, or tasks, but for a single fixed loss function. 
We generalize both lines of work and show that simultaneous adaptability to many downstream losses \textit{and} tasks is readily attainable. 
Figure~\ref{fig:flow} illustrates the relationship between these concepts. 

Our approach yields several benefits over prior works. 
Quantitatively, we improve the sample complexity of deterministic omniprediction by a factor of $1/\varepsilon$, and match all other known sample complexities of omniprediction and multi-group learning. 
Qualitatively, we provide a unified algorithmic path to omniprediction and multi-group learning---which prior works have treated with bespoke approaches that had difficulty learning, e.g., deterministic predictors---and significantly simplify proofs. 

\subsection{Related Works}

\paragraph{Omniprediction} Gopalan et al. \cite{gopalan_omnipredictors_2022} initiated the study of omniprediction by showing that an appropriately calibrated predictor can be post-processed to minimize a range of convex and Lipschitz losses. 
Subsequent work deepened the conceptual connection between omniprediction and calibration by formalizing the notion of a predictor being indistinguishable from the Bayes predictor, as measured by a set of loss functions \cite{gopalan_loss_2022}. 
Using these insights, recent work improved the sample complexity and oracle efficiency of omniprediction by sidestepping full calibration and relying on statistically efficient notions of calibration \cite{okoroafor_near-optimal_2025}. 
We carefully extend this line of work to account for the more challenging setting where loss functions are evaluated over overlapping subgroups of the domain. 

\paragraph{Multi-group learning}
Motivated by subgroup fairness, Blum and Lykouris \cite{blum_advancing_2020} and Rothblum and Yona \cite{DBLP:conf/icml/RothblumY21} respectively initiated the study of multi-group learning in the sequential and batch settings.
A rich line of work has since refined sample complexity guarantees  \cite{DBLP:conf/icml/Tosh022} and studied variants with oracle-efficiency \cite{deng_group-wise_2024}, hierarchical group structure \cite{deng_multi-group_2024}, and robustness concerns \cite{ahmadi_agnostic_2024}. 
Typical multi-group algorithms leverage reductions to the sleeping experts problem. In contrast, we leverage a fundamental connection between multi-group learning and calibration. 

\paragraph{Calibration}
Calibration originates from the online forecasting literature \cite{dawid1982well, foster_asymptotic_1998}, where it was proposed as a basic sanity check for predictions. 
More recently, Hebert-Johnson et al. \cite{hebert-johnson_multicalibration_2018} initiated the study of multicalibration, which connects calibration to subgroup robustness. 
Subsequent work developed interpretations of calibration as requiring that a learned predictor be indistinguishable from the Bayes predictor according to a class of tests \cite{dwork_outcome_2021}, and variants of calibration that provide downstream decision-theoretic guarantees \cite{KLST23, qiao_truthfulness_2025}.
Multi-objective learning is a flexible algorithmic framework for multicalibration \cite{DBLP:conf/nips/HaghtalabJ023} and related multi-distribution learning problems \cite{zhang2023optimal}.
Our techniques build on these works, especially ideas from indistinguishability and algorithms from multi-objective learning.

\section{Models and Preliminaries}
\label{sec:prelim}
We study batch prediction with binary labels. 
Let $D$ be a joint distribution over context space $\cX$ and binary labels $\cY = \{0, 1\}$. 
Let $\cH \subset \{ h: \cX \to \hat{\cY} \}$ be a class of hypothesis functions, where $\hat \cY$ is the prediction space. 
We study both the binary prediction setting, where $\hat{\cY} = \{0, 1\}$, and the probabilistic prediction setting, where $\hat{\cY} = [0, 1]$. 
Let $\cP = \{ h: \cX \to [0, 1] \}$ be the class of all real-valued hypothesis functions.
Denote the simplex over $\cH$ and $\cP$ as $\Delta(\cH)$ and $\Delta(\cP)$, respectively. 
The simplex in $\mathbb{R}^d$ is denoted by $\Delta^{d-1}$.
For tractability, each real-valued hypothesis is quantized to return predictions on the $\lambda$-net of the unit interval, $I_\lambda = \{0, \lambda, 2\lambda, \dots, 1 \}$, where the quantization parameter $\lambda \in (0, 1)$ is chosen by the learner with knowledge of the target error tolerance $\varepsilon$. 
Let $\cL \subset \{ \ell: \hat{\cY} \times \cY \to [-1, 1] \}$ be a class of loss functions that only take the prediction and label as inputs. 
Let $\cG \subseteq 2^\cX$ be a set of (possibly overlapping) groups on $\cX$, with each group identified by a group membership function $g: \cX \to \{0, 1\}$. 
Overloading notation, we say $x \in \cX$ belongs to group $g \in \cG$ if $g(x) = 1$. 
We denote group size under $D$ by $P_g = \Pr_{(x, y) \sim D}(g(x)=1)$. 

\subsection{Panprediction}
We formalize the problem of constructing a flexible predictor that downstream decision makers can adapt to minimize many different loss functions on many different tasks. 
To obtain this flexibility, we require that the universe of all losses, tasks (represented by groups over the domain), and competitor hypotheses of interest to downstream decision makers be pre-specified and collected into the classes $(\cL, \cG, \cH)$.
Given this triplet, the goal of panprediction is to use samples from the distribution $D$ to construct a predictor that, for any post-hoc choice of loss $\ell \in \cL$ and group $g \in \cG$, can be post-processed to compete with the best hypothesis $h \in \cH$ on the group $g$, as measured by the loss $\ell$.
We make the following assumptions about $(\cL, \cG, \cH)$.

First, we place a regularity condition on $\cL$. The total variation of a function $f: [0, 1] \to [-1, 1]$ is defined as
\begin{align*}
    V(f) = \sup_{m \in \mathbb{N}} \, \sup_{0 = z_0 < \dots < z_m = 1} \sum_{j=1}^{m} \abs{f(z_j) - f(z_{j-1})}.
\end{align*}
\begin{assumption}
    All losses $\ell \in \cL$ have bounded variation in the first argument, meaning
    \begin{align*}
        \cL \subseteq \cL_{\text{BV}} \triangleq \left\{ \ell: \sup_{y \in \cY} V(\ell(\cdot, y)) \leq 1 \right\}.
    \end{align*}
\label{as:bv}
\end{assumption}
This is a very mild restriction: $\cL_{\text{BV}}$ has infinite cardinality and subsumes all standard loss functions---zero-one loss, the hinge loss, square loss, and pinball loss---as well as $1$-Lipschitz functions and proper scoring rules (up to re-scaling). 
Losses having bounded variation is the weakest assumption made in omniprediction \cite{okoroafor_near-optimal_2025}, a related problem setting discussed in Section~\ref{sec:connect}.

Next, we mildly constrain the expressivity of $\cH$ and $\cG$.

\begin{assumption}
    \!$\cH$\! has finite combinatorial dimension.
\label{as:hdim}
\end{assumption}
If $\cH$ is a class of binary predictors, Assumption~\ref{as:hdim} requires that $\cH$ has bounded VC dimension. 
If $\cH$ is real-valued, then Assumption~\ref{as:hdim} requires that $\cH$ has bounded pseudo-dimension. 
Both are standard assumptions in statistical learning theory. 

\begin{assumption}
    $\cG$ has finite VC dimension. 
\label{as:gdim}
\end{assumption}
Intuitively, Assumption~\ref{as:gdim} allows groups to have infinite cardinality and encode rich structure, but requires that they are at least \textit{learnable} from data. 

Given Assumptions \ref{as:bv}, \ref{as:hdim}, and \ref{as:gdim}, we can formally define the desiderata of panprediction.

\begin{definition}[Deterministic Panprediction]
    Given loss class $\cL$, hypothesis class $\cH$, and set of groups $\cG$, a \textit{deterministic} predictor $p^*: \cX \to [0, 1]$ is a $(\cL, \cG, \cH, \eps)$-panpredictor if for all $\ell \in \cL$ and $g \in \cG$,
    \begin{align*}
        &\EEsc{(x, y) \sim D}{\ell(k_\ell (p^*(x)), y) }{g(x) = 1} \leq \min_{h \in \cH} \EEsc{(x, y) \sim D}{\ell(h(x), y)}{g(x)=1} + \eps \cdot \sqrt{P_g^{-1}},
    \end{align*}
    where the post-processing function $k_\ell: [0, 1] \to \hat{\cY}$ is
    \begin{equation}\label{eq:postprocess}
        k_\ell(p) = \argmin_{\hat{y} \in \hat{\cY}} \EEs{y \sim \mathrm{Ber}(p)}{\ell(\hat{y}, y)} = \argmin_{\hat{y} \in \hat{\cY}} \left( p \cdot \ell(\hat{y}, 1) + (1-p) \cdot \ell(\hat{y}, 0) \right).
    \end{equation}
\label{def:detpan}
\end{definition}
The post-processing step involves solving a simple optimization problem, which downstream decision makers can do with no access to $D$, and often in closed form. 

We can also define a less restrictive notion of panprediction that allows the predictor $p^*$ to randomize.
\begin{definition}[Randomized Panprediction]
    Given loss class $\cL$, hypothesis class $\cH$, and set of groups $\cG$, a \textit{randomized} predictor $\bp^* \in \Delta(\cP)$ is a $(\cL, \cG, \cH, \eps)$-panpredictor if for all $\ell \in \cL$ and  $g \in \cG$,
    \begin{align*}
        &\EEsc{(x, y) \sim D, \, p^* \sim \bp^*}{\ell(k_\ell (p^*(x)), y) }{g(x) = 1} \leq \min_{h \in \cH} \EEsc{(x, y) \sim D}{\ell(h(x), y) }{g(x)=1} + \eps \cdot \sqrt{P_g^{-1}}.
    \end{align*}
\label{def:randpan}
\end{definition}

\begin{remark}[Deterministic vs Randomized]
    All things equal, we prefer deterministic predictors to randomized ones---though they can be more sample-intensive to learn. In the sequel, we state most definitions and results for deterministic predictors, and defer a formal statement of the randomized variants to Appendix~\ref{sec:prelim-proofs}. 
\end{remark}

\begin{remark}[Conditional Guarantees]
    We define panprediction risk as a group-conditional expectation, and error tolerance as scaling inversely with the square root of each group's size under $D$. 
    Both are strong guarantees sought in multi-group learning and multicalibration \cite{DBLP:conf/icml/Tosh022, DBLP:conf/nips/HaghtalabJ023}. The choice of error tolerance scaling, in particular, reflects the optimal sample complexity of learning each group using samples from $D$. 
\end{remark}

All our definitions of panprediction are non-vacuous, since the Bayes predictor $\EEsc{}{y}{x}$ is a panpredictor for \textit{all} $\cL$, $\cG$, and $\cH$. 
Consequently, a panpredictor can be interpreted as a coarsening of the Bayes predictor with respect to a particular set of losses, groups, and hypotheses. 
We make this notion of ``coarsening" precise in the following discussion of step calibration. 

\subsection{Step Calibration}
Calibration---which requires that a probabilistic predictor is unbiased on its own level set---is the key property of the Bayes predictor that makes it useful for a wide range of downstream decision makers. 
Calibration is a sufficient condition for decision making, in the sense that decision makers that treat \textit{any} calibrated predictor as the Bayes predictor enjoy provable loss minimization guarantees \cite{foster_calibrated_1997, gopalan_omnipredictors_2022}.
However, full calibration can be more statistically expensive than direct loss minimization, and recent work has explored efficiently attainable notions of calibration that retain decision-theoretic guarantees \cite{KLST23}.

One such notion is \emph{step calibration}~\cite{qiao_truthfulness_2025}, also known as proper calibration \cite{okoroafor_near-optimal_2025}.
In contrast to full calibration, step calibration only requires that a predictor is unbiased on its \textit{sublevel} set. 
That is, on the set of all instances whose predicted probability is at most a given threshold. 
Below, we introduce a multi-objective variant of step calibration that not only holds marginally, but also on carefully defined subsets of the domain. 

\begin{definition}[$(\cG, \cH, \eps)$-Step Calibration]
    A deterministic predictor $p^*: \cX \to [0, 1]$ is $(\cG, \cH, \eps)$-step calibrated if for all $v, w \in [0, 1]$, $h \in \cH$, and $g \in \cG$,
    \begin{align*}
        &\abs{\EEsc{D}{(y - p^*(x)) \cdot \mathbbm{1}[p^*(x) \leq v, h(x) \leq w]}{g(x) = 1}} \leq \eps \cdot \sqrt{P_g^{-1}}.
    \end{align*}
\label{def:detstep}
\end{definition}

In addition to serving as the indicator for sublevel sets of $p^*$, the step function $\mathbbm{1}[p^*(x) \leq v]$ serves an approximation-theoretic purpose, as a natural basis for bounded variation functions.
This connection is further discussed in Section~\ref{sec:pan}. Also notice that in \Cref{def:detstep}, we condition on the event $\{ g(x) = 1 \}$, as in the formulation of panprediction risk, and take an indicator on the events $\{p^*(x) \leq v\}$ and $\{h(x) \leq w \}$. 

We sometimes invoke a weaker notion of unbiasedness, called multiaccuracy, that does not take an indicator on the (sub)level sets of the predictor. 
\begin{definition}[$(\cG, \cH, \eps)$-Multiaccuracy]
    A deterministic predictor $p^*: \cX \to [0, 1]$ is $(\cG, \cH, \eps)$-multiaccurate if for all $w \in [0, 1]$, $h \in \cH$, and $g \in \cG$,    \begin{align*}
        &\abs{\EEsc{(x, y) \sim D}{(y - p^*(x)) \cdot \mathbbm{1}[h(x) \leq w]}{g(x)=1}} \leq \eps \cdot \sqrt{P_g^{-1}}.
    \end{align*}
\label{def:detma}
\end{definition}

Randomized variants of \Cref{def:detstep} and \ref{def:detma} are deferred to Appendix~\ref{sec:prelim-proofs}.

\section{Reducing Panprediction to Step Calibration}
\label{sec:pan}
We now show that panprediction admits a clear reduction to step calibration by carefully extending the language of outcome indistinguishability, which was originally developed for omniprediction \cite{gopalan_loss_2022, dwork_outcome_2021}. 
In this section, we assume $\hat{\cY} = [0, 1]$. 
The setting with $\hat{\cY} = \{0, 1\}$ is strictly easier and follows readily from our arguments. 

\subsection{Loss Outcome Indistinguishability}
The Loss OI framework \cite{gopalan_loss_2022} is useful for deducing sufficient conditions for panprediction guarantees to hold. 
For \textit{any} deterministic predictor $p^*: \cX \to [0, 1]$, by the definition of the post-processing function $k_\ell$, the following inequality holds: for all $\ell \in \cL$, $g \in \cG$, $h \in \cH$,
\begin{align}
    &\EEsc{\substack{x \sim D \\ \tilde{y} \sim \mathrm{Ber}(p^*(x))}}{\ell(k_\ell (p^*(x)), \tilde{y}) }{g(x)=1} \leq \EEsc{\substack{x \sim D \\ \tilde{y} \sim \mathrm{Ber}(p^*(x))}}{\ell(h(x), \tilde{y}) }{g(x)=1}.
\label{eq:ideal}
\end{align}

To derive the deterministic $(\cL, \cG, \cH, O(\eps))$-panpredictor guarantee from here, it suffices to show that $p^*$ approximates the true Bayes predictor well in some sense. 
Concretely, for all $\ell \in \cL$, $g \in \cG$, and $h \in \cH$, we want the following inequalities to hold:
\begin{equation}\label{eq:dec1}
\nonumber
\begin{aligned}
\bigg|
  &\EEsc{(x,y)\sim D}{\ell(k_\ell(p^*(x)),y)}{g(x)=1} - \EEsc{\substack{x\sim D\\ \tilde y\sim \mathrm{Ber}(p^*(x))}}{\ell(k_\ell(p^*(x)),\tilde y)}{g(x)=1} \bigg| \le \frac{O(\eps)}{\sqrt{P_g}},
\end{aligned}
\end{equation}

\begin{equation}\label{eq:hyp1}
\nonumber
\begin{aligned}
\bigg|
  &\EEsc{(x,y)\sim D}{\ell(h(x),y)}{g(x)=1} - \EEsc{\substack{x\sim D\\ \tilde y\sim \mathrm{Ber}(p^*(x))}}{\ell(h(x),\tilde y)}{g(x)=1}\bigg| \le \frac{O(\eps)}{\sqrt{P_g}}.
\end{aligned}
\end{equation}

The first condition is called Decision OI, and the second condition, Hypothesis OI. To characterize each in terms of familiar terms, take the following definition.
\begin{definition}[Discrete Derivative]
    Given a loss $\ell \in \cL$, define the discrete derivative
    \begin{align*}
        \Delta \ell(p) = \ell(p, 1) - \ell(p, 0).
    \end{align*}
    Given a loss class $\cL$, let $\Delta \cL = \{ \Delta \ell \mid \ell \in \cL \}$ be the class of corresponding discrete derivatives.
\end{definition}
Intuitively, $\Delta \ell(p)$ quantifies how sharply $\ell$ distinguishes between $y = 1$ and $y = 0$ given a prediction $p \in [0, 1]$. 
Then Decision and Hypothesis OI can be respectively written as
\begin{align*}
    \abs{ \EEsc{D}{(y - p^*(x)) \cdot \Delta \ell( k_\ell(p^*(x)))}{g(x)=1} } &\leq \frac{O(\eps)}{\sqrt{P_g}}, \\
    \abs{ \EEsc{D}{(y - p^*(x)) \cdot \Delta \ell( h(x))}{g(x)=1} } &\leq \frac{O(\eps)}{\sqrt{P_g}}.
\end{align*}

This reformulation suggests that Decision OI can be deduced from a certain calibration condition, and Hypothesis OI, from a certain multiaccuracy condition. 

\subsection{Reducing Deterministic Panprediction to Determnistic Step Calibration}
Now we show that deterministic $(\cG, \cH, \eps)$-step calibration implies both sufficient conditions found above. 
\begin{theorem}
    If a deterministic predictor $p^*: \cX \to [0, 1]$ is $(\cG, \cH, \eps)$-step calibrated, then $p^*$ is a deterministic $(\cL, \cG, \cH, O(\eps))$-panpredictor.
\label{thm:detpan}
\end{theorem}

We prove Theorem~\ref{thm:detpan} by decomposing the step calibration guarantee into a calibration and a multiaccuracy guarantee (Lemma~\ref{lem:decomp}), then showing that the calibration guarantee implies Decision OI (Lemma~\ref{lem:decoi}), and that the multiaccuracy guarantee implies Hypothesis OI (Lemma~\ref{lem:hypoi}).

\begin{restatable}{lemma}{decomp}\label{lem:decomp}
   If a deterministic predictor $p^*: \cX \to [0, 1]$ is $(\cG, \cH, \eps)$-step calibrated, then $p^*$ is $(\cG, \emptyset, \eps)$-step calibrated and $(\cG, \cH, \eps)$-multiaccurate.  
\end{restatable}

The proof of Lemma~\ref{lem:decomp} is straightforward and is deferred to Appendix~\ref{sec:pan-proofs}. 
To prove the remaining two lemmas, it is useful to invoke the following technical result on approximating (the discrete derivative of) bounded variation functions with step functions.
\begin{restatable}{lemma}{approx}\label{lemma:approx}
    Suppose $\ell \in \cL_{\mathrm{BV}}$. Then for any predictor $f: \cX \to [0, 1]$ and $g \in \cG$,
    \begin{align*}
        &\abs{ \EEsc{(x, y) \sim D}{(y-p^*(x)) \cdot \Delta \ell(f(x))}{g(x)=1}} \leq 9 \sup_{v \in [0, 1]} \abs{ \EEsc{D}{(y-p^*(x)) \cdot \mathbbm{1}[f(x) \leq v]}{g(x)=1} } + \eps \cdot \sqrt{P_g^{-1}}.
    \end{align*}
\end{restatable} 
Lemma~\ref{lemma:approx} can be invoked with $f = p^*$ or $f = h$, for any $h\in \cH$, and helps upper bound expressions arising from Decision and Hypothesis OI in terms of step calibration error. 
Its proof is deferred to Appendix~\ref{sec:pan-proofs}.

\begin{lemma}
If a deterministic predictor $p^*$ is $(\cG, \emptyset, \eps)$-step calibrated, then for all $g \in \cG$,
    \begin{align*}
        \abs{ \EEsc{D}{(y - p^*(x)) \cdot \Delta \ell(k_\ell(p^*(x)))}{g(x)=1} } \leq \frac{O(\eps)}{\sqrt{P_g}}.
    \end{align*}
    \vspace{-\baselineskip}
\label{lem:decoi}
\end{lemma}

\begin{proof}[Proof of Lemma~\ref{lem:decoi}]
    For any loss $\ell: [0, 1] \times \cY \to [-1, 1]$ and constants $p, q \in [0, 1]$, by definition of $k_\ell$, 
    \begin{align*}
        \EEs{y \sim \mathrm{Ber}(p)}{\ell(k_{\ell}(p), y)} \leq \EEs{y \sim \mathrm{Ber}(p)}{\ell(k_{\ell}(q), y)}.
    \end{align*}
    Hence,  $\ell(k_\ell(\cdot), y): [0, 1] \to [-1, 1]$ is a proper scoring rule and is contained in $\cL_{\mathrm{BV}}$. Then by invoking Lemma~\ref{lemma:approx} with $f = p^*$, we have that
    \begin{align*}
        \abs{ \EEsc{D}{(y - p^*(x)) \cdot \Delta \ell(k_\ell(p^*(x)))}{g(x)=1} } &\leq 9 \sup_{v \in [0, 1]} \abs{ \EEsc{D}{(y - p^*(x)) \cdot \mathbbm{1}[p^*(x) \leq v]}{g(x) = 1} } + \eps \cdot \sqrt{P_g^{-1}} \\
        &\leq O(\eps) \cdot \sqrt{P_g^{-1}},
    \end{align*}
    where the last inequality follows from the $(\cG, \emptyset, \eps)$-step calibration guarantee. 
\end{proof}

\begin{lemma}
    If a deterministic predictor $p^*$ is $(\cG, \cH, \eps)$-multiaccurate, then for all $h \in \cH$, $g \in \cG$,
    \begin{align*}
        \abs{ \EEsc{D}{(y - p^*(x)) \cdot \Delta \ell(h(x))}{g(x)=1} } \leq \frac{O(\eps)}{\sqrt{P_g}}.
    \end{align*}
    \vspace{-\baselineskip}
\label{lem:hypoi}
\end{lemma}

\begin{proof}[Proof of Lemma~\ref{lem:hypoi}]
     By assumption, $\ell \in \cL_{\mathrm{BV}}$. Then by applying Lemma~\ref{lemma:approx} with $f = h$ for any $\in \cH$,
     \begin{align*}
        \abs{ \EEsc{D}{(y - p^*(x)) \cdot \Delta \ell(h(x))}{g(x)=1} } &\leq 9 \sup_{w \in [0, 1]} \abs{ \EEsc{D}{(y - p^*(x)) \cdot \mathbbm{1}[h(x) \leq w]}{g(x) = 1} } + \eps \cdot \sqrt{P_g^{-1}} \\
        &\leq O(\eps) \cdot \sqrt{P_g^{-1}},
    \end{align*}
    where the last inequality follows from the $(\cG, \cH, \eps)$-multiaccuracy guarantee.
\end{proof}

\begin{proof}[Proof of Theorem~\ref{thm:detpan}]
    Starting with Equation~\ref{eq:ideal} and invoking Lemmas~\ref{lem:decomp}, ~\ref{lem:decoi}, and ~\ref{lem:hypoi} to swap terms on both sides of the inequality, at the cost of additive error $O(\eps \cdot \sqrt{P_g^{-1}})$, immediately yields the result. 
\end{proof}

\subsection{Extensions for Randomized Panprediction}
By extending the Loss OI machinery to work with the definitions randomized step calibration and randomized multiaccuracy, we can also show that randomized step calibration implies randomized panprediction. 
The statement and proof of the randomized variants of results are deferred to Appendix~\ref{sec:pan-proofs}.

\section{Step Calibration Algorithms via Multi-objective Learning}
\label{sec:algo}
In this section, we derive algorithms for deterministic and randomized step calibration using the multi-objective learning framework. 
Our deterministic step calibration algorithm improves on the best known sample complexity guarantee of deterministic omniprediction by a factor of $1/\eps$. 
Moreover, the sample complexity of our algorithms match the best known sample complexity guarantees of deterministic and randomized multi-group learning, up to logarithmic factors, demonstrating that panprediction guarantees can be attained ``for free'' whenever multi-group guarantees are sought. 

\subsection{Multi-objective Learning}
Fix a distribution $D$, a hypothesis class $\cF$, and a set of bounded objectives $\cO = \{ \ell: \cF \times \cX \times \cY \to [a, b] \}$. 
The goal of the $(D, \cO, \cH)$-multi-objective learning framework, formalized by Haghtalab, Jordan, and Zhao~\cite{DBLP:conf/nips/HaghtalabJ023}, is finding a deterministic predictor $f^* \in \cF$ such that  
\begin{equation}\label{eq:multiobj}
       \max_{\ell \in \cO} \EEs{(x, y) \sim D}{\ell(f^*(x), y)} \leq \min_{f \in \cF} \max_{\ell \in \cO} \EEs{(x, y) \sim D}{\ell(f(x), y)} + \eps.
\end{equation}
A randomized predictor $\bh^* \in \Delta(\cH)$ can also be found, with the expectation on the left hand side of Equation~\ref{eq:multiobj} holding on average over the draw of $h^* \sim \bh^*$. 
Below, we show that the $(\cG, \cH, \eps)$-step calibration problem can be formulated as a multi-objective problem. 
In the rest of the paper, we assume that the group probabilities $\{ P_g \}_{g \in \cG}$ are known constants, and denote $\gamma = \min_{g \in \cG} P_g$. 
This is a standard assumption in, e.g., multi-group learning \cite{DBLP:conf/icml/Tosh022}.

\begin{theorem}
    Fix a distribution $D$, the competitor hypothesis class $\cH$, and a set of groups $\cG$. For each $\sigma \in \{ \pm 1 \}$, $v, w \in [0, 1]$, $h \in \cH$, and $g \in \cG$, define the objective $\ell_{\sigma, v, w, h, g}: \cP \times \cX \times \cY \to [-1/\sqrt{\gamma}, 1/\sqrt{\gamma}]$ as
    \begin{align*}
        &\ell_{\sigma, v, w, h, g}(p, (x, y)) = \frac{\sigma \cdot \left(y - p(x) \right)}{\sqrt{P_g}} \cdot \mathbbm{1}[p(x) \leq v, h(x) \leq w, g(x) = 1].
    \end{align*} 
    Let $\cO_{\mathrm{sc}} = \{ \ell_{\sigma, v, w, h, g} \}$ be the set of all such objectives. 
    If $p^* \in \cP$ is an $\eps$-optimal solution to the $(D, \cO_{\mathrm{sc}}, \cP)$-multi-objective learning problem, then $p^*$ is $(\cG, \cH, O(\eps))$-step calibrated.
\label{thm:multiobj}
\end{theorem}

\begin{proof}[Proof of Theorem~\ref{thm:multiobj}]
    By assumption that $p^*$ is an $\eps$-optimal solution to the $(\cD, \cO_{\mathrm{sc}}, \cP)$-multi-objective learning problem, we have that
    \begin{align*}
        \max_{\sigma, v, w, h, g}\EEs{(x, y) \sim D}{\ell_{\sigma, v, w, h, g}(p^*, (x, y))} &= \max_{v, w, h, g} \frac{\abs{\EEs{}{(y - p^*(x)) \cdot \mathbbm{1}[p^*(x) \leq v, h(x) \leq w, g(x) =1]}}}{\sqrt{P_g}} \\
        &\leq \min_{p \in \cP} \max_{v, w, h, g} \EEs{(x, y) \sim D}{\ell_{\sigma, v, w, h, g}(p, (x, y))} + \eps \\
        &= \eps,
    \end{align*}
    where the last equality holds because the expectation is zero when $p$ is the Bayes predictor. 
    Since
    \begin{align*}
        &\abs{\EEs{D}{(y - p^*(x)) \cdot \mathbbm{1}[p^*(x) \leq v, h(x) \leq w, g(x) = 1]}} = P_g \abs{\EEsc{D}{(y - p^*(x)) \mathbbm{1}[p^*(x) \leq v, h(x) \leq w]}{g(x) = 1}}
    \end{align*}
    it follows that for all $v, w \in [0, 1]$, $h \in \cH$, and $g \in \cG$,
    \begin{align*}
        &\abs{\EEsc{D}{(y - p^*(x)) \cdot \mathbbm{1}[p^*(x) \leq v, h(x) \leq w]}{g(x) = 1}} \leq \eps \cdot \sqrt{P_g^{-1}}.
        \qedhere
\end{align*}
\end{proof}

Theorem~\ref{thm:multiobj} establishes that step calibrated predictors correspond to the approximate equilibria of a particular two-player, zero-sum game. 
Following Haghtalab, Jordan, and Zhao \cite{DBLP:conf/nips/HaghtalabJ023}, we construct step calibration algorithms that use tools from online learning to solve this game.

A technical detail that must be handled before designing algorithms is the fact that $\cO_{\mathrm{sc}}$ can have infinite cardinality.
In order to make algorithms tractable, we leverage Assumptions~\ref{as:bv}, \ref{as:hdim}, and \ref{as:gdim} to construct a finite cover, denoted by $\hat{O}_\mathrm{sc}$.
This covering argument is standard in statistical learning theory, and details are deferred to Appendix~\ref{sec:algo-proofs}.

\subsection{Deterministic Step Calibration}

\begin{algorithm}[t!] 
\caption{Deterministic Step Calibration} \label{alg:detstep}
\begin{algorithmic}[1]
\Require $\epsilon, \delta \in (0,1)$, $T \in \mathbb{N}$, $c \in [0, 1]$, $C \in \mathbb{N}$, sampling access to $D$, best response oracle $\mathcal{A}$
\State Initialize Hedge iterate $p^{(1)} \;\gets\; [\tfrac 12,\tfrac 12]^{\mathcal{X}}$
\For{$t=1,\dots,T$}
    \State Update
    \begin{align*}
        \ell^{(t)} \;\gets\; \mathcal{A}\bigl(p^{(t)}, \hat{\cO}_{\mathrm{sc}}^{\gamma}, c \eps\sqrt{\gamma} \bigr),
    \end{align*}
     where $\hat{\cO}^\gamma_{\text{sc}}$ is a modification of $\cO_{\text{sc}}$ specified in Appendix~\ref{sec:algo-proofs}
    \State For each $x\in\mathcal{X}$, update
    \begin{align*}
        p^{(t+1)}(x)\;\gets\;\mathrm{Hedge}\bigl(c_x^{(1)}, \dots, c_x^{(t)}\bigr),
    \end{align*}
     where $c_x^{(t)}$ is a modification of $\ell^{(t)}$ specified in Appendix~\ref{sec:algo-proofs}
\EndFor
\State Draw $m\gets \tfrac{C\log(T/\delta)}{\epsilon^2}$ samples $z^{(1)}, \dots, z^{(m)} \sim D$
\State $t^*\;\gets\;\arg\min_{t\in[T]}\sum_{(x,y)\in z^{(1:m)}}\ell^{(t)}\bigl(p^{(t)},(x,y)\bigr)$
\State \Return $p^{(t^*)}$
\end{algorithmic}
\end{algorithm}

Below, we build intuition about our algorithms and state guarantees. 
Technical details are deferred to Appendix \ref{sec:algo-proofs}. 

Per Haghtalab, Jordan, and Zhao \cite{DBLP:conf/nips/HaghtalabJ023}, deterministic solutions to the multiobjective problem can be found via ``no-regret vs best-response'' dynamics between a fictitious player, who constructs a sequence of predictors using a no-regret algorithm, and a fictitious adversary, who constructs a sequence of objectives using a best-response oracle.
In our setting, the player is instantiated with a copy of the Hedge algorithm \cite{freund_decision-theoretic_1997} per point $x \in \cX$. 
The player's prediction at $x \in \cX$ at iteration $t$ is denoted by $p^{(t)}(x)$ and evolves according to the Hedge update rule, which takes (a light modification of) the historical loss functions $\ell^{(1)}, \dots, \ell^{(t-1)}$ as input.
In turn, at each iteration, the adversary selects a loss $\ell^{(t)} \in \cO_{\mathrm{sc}}$ that approximately maximizes penalty against the player's Hedge algorithm.
That is, given predictor $p^{(t)}$, objectives $\cO_{\mathrm{sc}}$, and error tolerance $\eps'$, $\ell^{(t)} \in \cO_{\mathrm{sc}}$ is chosen such that 
\begin{align*}
    &\EEs{(x, y) \sim D}{\ell(p^{(t)}, (x, y))} \geq \max_{\ell^* \in \cO_{\mathrm{sc}}} \EEs{(x, y) \sim D}{\ell^*(p^{(t)}, (x, y))} - \eps'.
\end{align*}
The above computation is abstracted into a best-response oracle $\cA$. 
This oracle requires samples from $D$ to approximate the expectation, and uses techniques from adaptive data analysis~\cite{bassily_algorithmic_2015} in order to reduce sample complexity overhead.
Putting everything together in Algorithm~\ref{alg:detstep}, we have the following result.

\begin{restatable}{theorem}{detstep}\label{thm:detstep}
     Fix $\varepsilon > 0$ and $\delta \in (0, 1)$. Let $d_H$ be the appropriate combinatorial dimension of $H$ and $d_G$ the VC dimension of $\cG$. Then with probability at least $1-\delta$ and at most $\tilde{O} \left( \eps^{-3} \cdot (d_H + d_G) \cdot \log(1/\delta) \right)$ samples from $D$, Algorithm~\ref{alg:detstep} returns a deterministic predictor $p^*$ that is $(\cG, \cH, \eps)$-step calibrated.
\end{restatable}

The proof of Theorem~\ref{thm:detstep} is deferred to Appendix~\ref{sec:algo-proofs}. 
Notably, the $\tilde{O}\left( \eps^{-3} \cdot d_H \right)$ sample complexity guarantee improves the best known sample complexity guarantee of $\tilde{O}\left( \eps^{-4} \cdot d_H \right)$ for deterministic omniprediction \cite{okoroafor_near-optimal_2025} by a factor of $\eps^{-1}$. 
Moreover, it matches the best known sample complexity guarantee for deterministic multi-group learning \cite{DBLP:conf/icml/Tosh022} up to logarithmic factors. 

\subsection{Randomized Step Calibration}
Per Haghtalab, Jordan, and Zhao \cite{DBLP:conf/nips/HaghtalabJ023}, randomized solutions to the multi-objective problem can be found via ``no-regret vs no-regret'' dynamics. 
In our setting, this corresponds to using instantiations of Hedge to construct the predictor $p^{(t)}$ and another instantiation of Hedge to select the losses $\ell^{(t)}$ on which the predictor is updated. 
An advantage of this dynamic, over the no-regret vs best-response dynamics considered for deterministic multi-objective learning, is that it only incurs a sample complexity of $\tilde{O}(1/\epsilon^2)$.

\begin{algorithm}[tb]
\caption{Randomized Step Calibration}
\label{alg:randstep}
\begin{algorithmic}[1]
\Require $\epsilon, \delta \in (0,1)$, $T \in \mathbb{N}$, $c \in [0, 1]$, $C \in \mathbb{N}$, sampling access to $D$
\State Initialize Hedge iterates $p^{(1)} \;\gets\; [\tfrac 12,\tfrac 12]^{\mathcal{X}}$ and $q^{(1)} = \mathrm{Uniform}(\hat{\cO}_{\mathrm{sc}}^{\gamma})$.
\For{$t=1,\dots,T$}
    \State Sample objective $\ell^{(t)} \sim q^{(t)}$ and data point $(x^{(t)}, y^{(t)}) \sim D$
    \State For each $x\in\mathcal{X}$, update 
    \begin{align*}
        p^{(t+1)}(x)\;\gets\;\mathrm{Hedge}\bigl(c_x^{(1)}, \dots, c_x^{(t)}\bigr),
    \end{align*}
    where $c_x^{(t)}(\hat y)$ is a modification of $\ell^{(t)}$ as specified in Appendix~\ref{sec:algo-proofs}
    \State Let $q^{(t+1)} = \mathrm{Hedge}(c_{\mathrm{adv}}^{(1)}, \dots, c_{\mathrm{adv}}^{(t)})$ where $c_{\mathrm{adv}}^{(t)} = 1-\ell_{(\cdot)}(p^{(t)}, (x^{(t)}, y^{(t)}))$
\EndFor
\State \Return $\bp^{*}$, a uniform distribution over $p^{(1)}, \dots, p^{(T)}$
\end{algorithmic}
\end{algorithm}

\begin{restatable}{theorem}{randstep}\label{thm:randstep}
    Fix $\varepsilon > 0$ and $\delta \in (0, 1)$. Let $d_H$ be the appropriate combinatorial dimension of $H$ and $d_G$ the VC dimension of $\cG$. Then with probability at least $1-\delta$ and at most $\tilde{O} \left( \eps^{-2} \cdot (d_H + d_G) \cdot \log(1/\delta) \right)$ samples from $D$, Algorithm~\ref{alg:randstep} returns a randomized predictor $p^*$ that is $(\cG, \cH,  \eps)$-step calibrated.
\end{restatable}

The statement of the algorithm and the proof of Theorem~\ref{thm:randstep} are deferred to Appendix~\ref{sec:algo-proofs}. 
Importantly, the above result matches the best known sample complexity guarantees of randomized omniprediction \cite{okoroafor_near-optimal_2025} and multi-group learning \cite{DBLP:conf/icml/Tosh022} up to logarithmic factors. 
Furthermore, the $\tilde{O}(1/\eps^2)$ rate is minimax optimal for minimizing one loss over one task. 
This establishes that under mild assumptions---such as allowing for randomized predictors---simultaneously minimizing infinitely many losses on infinitely many tasks can be as statistically easy as minimizing one loss on one task. 
This strengthens the conclusion of Okoroafor, Kleinberg, and Kim \cite{okoroafor_near-optimal_2025}, who showed that simultaneously minimizing infinitely many losses can be as statistically easy as minimizing one loss. 

\section{Connections to Omniprediction and Multi-group Learning}
\label{sec:connect}
In this section, we make precise how panprediction generalizes the problems of \emph{omniprediction}~\cite{gopalan_omnipredictors_2022, okoroafor_near-optimal_2025} and \emph{multi-group learning}~\cite{DBLP:conf/icml/RothblumY21, DBLP:conf/icml/Tosh022}, which respectively study predictions that generalize to many losses or many tasks, but not both.

\subsection{Omniprediction}
The goal of omniprediction is to use samples from $D$ to construct a predictor that, for any post-hoc choice of loss $\ell \in \cL$, can be post-processed to compete with the best hypothesis $h \in \cH$, as measured by the loss $\ell$.

\begin{definition}[Deterministic Omniprediction]
    Given loss class $\cL$ and hypothesis class $\cH$, a \textit{deterministic} predictor $p^*: \cX \to [0, 1]$ is a $(\cL, \cH, \eps)$-omnipredictor if for each $\ell \in \cL$, 
    \begin{align*}
        \EEs{(x, y) \sim D}{\ell(k_\ell ( p^*(x)), y)} \leq \min_{h \in \cH} \EEs{(x, y) \sim D}{\ell( h(x), y)} + \eps.
    \end{align*}
    The post-processing $k_\ell$ is defined as in Equation~\ref{eq:postprocess}.
\end{definition}

By symmetry in the definitions, a deterministic panpredictor with the trivial group over the entire domain, $\cG = \{ \cX \}$, is a deterministic omnipredictor. 

\begin{proposition}
    Let $\cL$ be a loss class and $\cH$ a hypothesis class. If the deterministic predictor $p^*: \cX \to [0, 1]$ is a $(\cL, \{\cX\}, \cH, \eps)$-panpredictor, then $p^*$ is a deterministic $(\cL, \cH, \eps)$-omnipredictor.
\end{proposition}

The definition of randomized omniprediction, and a corresponding result that a randomized panpredictor with $\cG = \{ \cX \}$ is a randomized omnipredictor, are deferred to Appendix~\ref{sec:connect-proofs}.

\subsection{Multi-group Learning}
The goal of multi-group learning is to use samples from $D$ to construct a predictor that, on each group $g \in \cG$, competes with best hypothesis $h \in \cH$ for that group, as measured by the zero-one loss.

\begin{definition}[Deterministic MG Learning]
    Given the zero-one loss $\ell$, set of groups $\cG$, and binary hypothesis class $\cH$, a \textit{deterministic} hypothesis $h^*: \cX \to \{0, 1\}$ is a $(\ell, \cG, \cH, \eps)$-multi-group learner if for each group $g \in \cG$, 
    \begin{align*}
        &\EEsc{(x, y) \sim D}{\ell(h^*(x), y)}{g(x) = 1} \leq \min_{h \in \cH} \EEsc{(x, y) \sim D}{\ell(h(x), y)}{g(x) = 1} + \eps \cdot \sqrt{P_g^{-1}}.
    \end{align*}
\end{definition}

Panprediction sits upstream from multi-group learning, in the sense that the post-processing of an appropriate panpredictor is a multi-group learning solution. 

\begin{restatable}{proposition}{detmulti}\label{prop:detmulti}
    Fix the zero-one loss $\ell$, groups $\cG$, and binary hypothesis class $\cH$. If a deterministic predictor $p^*: \cX \to [0, 1]$ is a $(\{\ell\}, \cG, \cH, \eps)$-panpredictor, then the binary classifier $h(x) = \mathbbm{1}[p^*(x) \geq 0.5]$ is a deterministic $(\ell, \cG, \cH, \eps)$-multi-group learner.
\end{restatable}

The definition of randomized multi-group learning, and a corresponding result that the post-processing of an appropriate randomized panpredictor is a randomized multi-group learning solution, are deferred to Appendix~\ref{sec:connect-proofs}.

\section{Discussion}
We formalize the problem of panprediction in the binary prediction setting and design statistically efficient panpredictors via a reduction to step calibration. 
Important questions for future work include:
\begin{itemize}
    \item Is the $\varepsilon^{-1}$ sample complexity gap between deterministic and randomized panprediction fundamental? 
    This question also remains open in both omniprediction \cite{okoroafor_near-optimal_2025} and multi-group learning \cite{DBLP:conf/icml/Tosh022}.
   
    \item What is the sample complexity of oracle-efficient panprediction, and is there a gap compared to the information-theoretic bounds we prove? 
    Similar questions on statistical-computational gaps are only just being explored in omniprediction \cite{okoroafor_near-optimal_2025} and multi-group learning \cite{deng_group-wise_2024}.
    
    \item Can multi-class panprediction be achieved with the same sample complexity as direct multi-class loss minimization? 
    This question also remains open in omniprediction. 
\end{itemize}

\section{Acknowledgements}
NH was supported in part by the National Science Foundation under grant CCF-2145898, by the Office
of Naval Research under grant N00014-24-1-2159, a Google Research Scholar Award, an Alfred P. Sloan fellowship, and a Schmidt Science AI2050 fellowship.
DH was supported by the Office of Navel Research under grant N00014-24-1-2700.
This work was partially done while the authors were visitors at the Simons Institute for
the Theory of Computing.

\newpage
\bibliographystyle{unsrt}
\bibliography{references}

\begin{thebibliography}{10}

\bibitem{DBLP:conf/icml/Tosh022}
Christopher~J. Tosh and Daniel Hsu.
\newblock Simple and near-optimal algorithms for hidden stratification and multi-group learning.
\newblock In Kamalika Chaudhuri, Stefanie Jegelka, Le~Song, Csaba Szepesv{\'{a}}ri, Gang Niu, and Sivan Sabato, editors, {\em International Conference on Machine Learning, {ICML} 2022, 17-23 July 2022, Baltimore, Maryland, {USA}}, volume 162 of {\em Proceedings of Machine Learning Research}, pages 21633--21657. {PMLR}, 2022.

\bibitem{okoroafor_near-optimal_2025}
Princewill Okoroafor, Robert Kleinberg, and Michael~P. Kim.
\newblock Near-{Optimal} {Algorithms} for {Omniprediction}, January 2025.
\newblock arXiv:2501.17205 [stat].

\bibitem{qiao_truthfulness_2025}
Mingda Qiao and Eric Zhao.
\newblock Truthfulness of {Decision}-{Theoretic} {Calibration} {Measures}, March 2025.

\bibitem{DBLP:conf/nips/HaghtalabJ023}
Nika Haghtalab, Michael~I. Jordan, and Eric Zhao.
\newblock A unifying perspective on multi-calibration: Game dynamics for multi-objective learning.
\newblock In Alice Oh, Tristan Naumann, Amir Globerson, Kate Saenko, Moritz Hardt, and Sergey Levine, editors, {\em Advances in Neural Information Processing Systems 36: Annual Conference on Neural Information Processing Systems 2023, NeurIPS 2023, New Orleans, LA, USA, December 10 - 16, 2023}, 2023.

\bibitem{zhang2023optimal}
Zihan Zhang, Wenhao Zhan, Yuxin Chen, Simon~S. Du, and Jason~D. Lee.
\newblock Optimal multi-distribution learning, 2023.

\bibitem{gopalan_omnipredictors_2022}
Parikshit Gopalan, Adam~Tauman Kalai, Omer Reingold, Vatsal Sharan, and Udi Wieder.
\newblock Omnipredictors.
\newblock In Mark Braverman, editor, {\em 13th {Innovations} in {Theoretical} {Computer} {Science} {Conference}, {ITCS} 2022, {January} 31 - {February} 3, 2022, {Berkeley}, {CA}, {USA}}, volume 215 of {\em {LIPIcs}}, pages 79:1--79:21. Schloss Dagstuhl - Leibniz-Zentrum für Informatik, 2022.

\bibitem{DBLP:conf/icml/RothblumY21}
Guy~N. Rothblum and Gal Yona.
\newblock Multi-group agnostic {PAC} learnability.
\newblock In Marina Meila and Tong Zhang, editors, {\em Proceedings of the 38th International Conference on Machine Learning, {ICML} 2021, 18-24 July 2021, Virtual Event}, volume 139 of {\em Proceedings of Machine Learning Research}, pages 9107--9115. {PMLR}, 2021.

\bibitem{gopalan_loss_2022}
Parikshit Gopalan, Lunjia Hu, Michael~P. Kim, Omer Reingold, and Udi Wieder.
\newblock Loss {Minimization} through the {Lens} of {Outcome} {Indistinguishability}, December 2022.
\newblock arXiv:2210.08649 [cs].

\bibitem{blum_advancing_2020}
Avrim Blum and Thodoris Lykouris.
\newblock Advancing {Subgroup} {Fairness} via {Sleeping} {Experts}.
\newblock {\em LIPIcs, Volume 151, ITCS 2020}, 151:55:1--55:24, 2020.
\newblock Artwork Size: 24 pages, 580511 bytes ISBN: 9783959771344 Medium: application/pdf Publisher: Schloss Dagstuhl – Leibniz-Zentrum für Informatik Version Number: 1.0.

\bibitem{deng_group-wise_2024}
Samuel Deng, Daniel Hsu, and Jingwen Liu.
\newblock Group-wise oracle-efficient algorithms for online multi-group learning.
\newblock {\em Advances in Neural Information Processing Systems}, 37:39462--39500, December 2024.

\bibitem{deng_multi-group_2024}
Samuel Deng and Daniel Hsu.
\newblock Multi-group {Learning} for {Hierarchical} {Groups}.
\newblock In {\em Proceedings of the 41st {International} {Conference} on {Machine} {Learning}}, pages 10440--10487. PMLR, July 2024.
\newblock ISSN: 2640-3498.

\bibitem{ahmadi_agnostic_2024}
Saba Ahmadi, Avrim Blum, Omar Montasser, and Kevin~M. Stangl.
\newblock Agnostic {Multi}-{Robust} {Learning} using {ERM}.
\newblock In {\em Proceedings of {The} 27th {International} {Conference} on {Artificial} {Intelligence} and {Statistics}}, pages 2242--2250. PMLR, April 2024.
\newblock ISSN: 2640-3498.

\bibitem{dawid1982well}
A~Philip Dawid.
\newblock The well-calibrated bayesian.
\newblock {\em Journal of the American Statistical Association}, 77(379):605--610, 1982.
\newblock Publisher: Taylor \& Francis.

\bibitem{foster_asymptotic_1998}
Dean~P. Foster and Rakesh~V. Vohra.
\newblock Asymptotic {Calibration}.
\newblock {\em Biometrika}, 85(2):379--390, 1998.
\newblock Publisher: [Oxford University Press, Biometrika Trust].

\bibitem{hebert-johnson_multicalibration_2018}
Ursula Hebert-Johnson, Michael Kim, Omer Reingold, and Guy Rothblum.
\newblock Multicalibration: {Calibration} for the ({Computationally}-{Identifiable}) {Masses}.
\newblock In {\em Proceedings of the 35th {International} {Conference} on {Machine} {Learning}}, pages 1939--1948. PMLR, July 2018.
\newblock ISSN: 2640-3498.

\bibitem{dwork_outcome_2021}
Cynthia Dwork, Michael~P. Kim, Omer Reingold, Guy~N. Rothblum, and Gal Yona.
\newblock Outcome indistinguishability.
\newblock In Samir Khuller and Virginia~Vassilevska Williams, editors, {\em {STOC} '21: 53rd {Annual} {ACM} {SIGACT} {Symposium} on {Theory} of {Computing}, {Virtual} {Event}, {Italy}, {June} 21-25, 2021}, pages 1095--1108. ACM, 2021.

\bibitem{KLST23}
Robert Kleinberg, Renato Paes~Leme, Jon Schneider, and Yifeng Teng.
\newblock U-calibration: Forecasting for an unknown agent.
\newblock In {\em Conference on Learning Theory (COLT)}, pages 5143--5145, 2023.

\bibitem{foster_calibrated_1997}
Dean~P. Foster and Rakesh~V. Vohra.
\newblock Calibrated {Learning} and {Correlated} {Equilibrium}.
\newblock {\em Games and Economic Behavior}, 21(1):40--55, October 1997.

\bibitem{freund_decision-theoretic_1997}
Yoav Freund and Robert~E Schapire.
\newblock A {Decision}-{Theoretic} {Generalization} of {On}-{Line} {Learning} and an {Application} to {Boosting}.
\newblock {\em Journal of Computer and System Sciences}, 55(1):119--139, August 1997.

\bibitem{bassily_algorithmic_2015}
Raef Bassily, Kobbi Nissim, Adam Smith, Thomas Steinke, Uri Stemmer, and Jonathan Ullman.
\newblock Algorithmic {Stability} for {Adaptive} {Data} {Analysis}, November 2015.
\newblock arXiv:1511.02513 [cs].

\bibitem{nemirovski_robust_2009}
A.~Nemirovski, A.~Juditsky, G.~Lan, and A.~Shapiro.
\newblock Robust {Stochastic} {Approximation} {Approach} to {Stochastic} {Programming}.
\newblock {\em SIAM Journal on Optimization}, 19(4):1574--1609, January 2009.
\newblock Publisher: Society for Industrial and Applied Mathematics.

\bibitem{mcsherry_mechanism_2007}
Frank McSherry and Kunal Talwar.
\newblock Mechanism {Design} via {Differential} {Privacy}.
\newblock In {\em 48th {Annual} {IEEE} {Symposium} on {Foundations} of {Computer} {Science} ({FOCS}'07)}, pages 94--103, October 2007.
\newblock ISSN: 0272-5428.

\bibitem{benedek_learnability_1991}
Gyora~M. Benedek and Alon Itai.
\newblock Learnability with respect to fixed distributions.
\newblock {\em Theoretical Computer Science}, 86(2):377--389, September 1991.

\bibitem{dai_learning_2024}
Jessica Dai, Nika Haghtalab, and Eric Zhao.
\newblock Learning {With} {Multi}-{Group} {Guarantees} {For} {Clusterable} {Subpopulations}, December 2024.

\end{thebibliography}

\newpage
\appendix 
\section{Deferred Definitions from Section~\ref{sec:prelim}} \label{sec:prelim-proofs}
\subsection{Quantizing Real-valued Hypotheses}
We say a hypothesis is real-valued if $h: \cX \to [0, 1]$.
For tractability, real-valued hypotheses are quantized so that they return predictions that are on the $\lambda$-net of the unit interval $[0, 1]$, denoted by $I_\lambda$. 
Importantly, $\lambda$ is a quantization parameter that the learner selects with full knowledge of the target error tolerance $\eps$. 
Formally, given a class of real-valued hypotheses $\cH$, we can define the quantized class
\begin{align*}
    \cH_\lambda = \{ h_\lambda: \cX \to I_\lambda \mid h_\lambda(x) = Q_\lambda(h(x)), \, h \in \cH \},
\end{align*}
where the ``nearest neighbor" quantization map $Q_\lambda: [0, 1] \to I_\lambda$ is defined as 
\begin{align*}
    Q_\lambda(p) = \arg\min_{p' \in I_\lambda} \abs{p - p'}.
\end{align*}
In all that follows, it suffices to set $\lambda = \Theta(\eps)$.

\subsection{Combinatorial Dimensions}
We recall the definitions of three  combinatorial dimensions we consider for the classes $\cH$ and $\cG$. 
For binary hypotheses, we consider the VC dimension. 
For real-valued hypotheses, we consider the pseudo-dimension.
Roughly, hypothesis classes with infinite cardinality but bounded combinatorial dimension admit finite covers with cardinality that is exponential in the combinatorial dimension. 

\begin{definition}[VC Dimension]
    Suppose $\cH$ is a class of binary hypotheses, i.e., $h: \cX \to \{0, 1\}$ for all $h \in \cH$. A finite set $S = \{x_1, \dots, x_m \} \subset \cX$ is shattered by $\cH$ if for every labeling $y \in \{0, 1\}^m$, there exists $h \in \cH$ with $h(x_i) = y_i$ for all $i \in [m]$. The VC dimension of $\cH$, denoted by $\mathrm{VC}(\cH)$, is the size of the largest set $S$ that is shattered by $\cH$. 
\end{definition}

\begin{definition}[Fat-Shattering Dimension]
    Suppose $\cH$ is a class of real-valued hypotheses, i.e., $h: \cX \to [0, 1]$ for all $h \in \cH$. Fix $\eps' > 0$. A finite set $S = \{x_1, \dots, x_m \} \subset \cX$ is $\eps'$-shattered by $\cH$ if there exist threshold values $w_1, \dots, w_m \in [0, 1]$ such that for every subset $S' \subseteq S$, there exists $h_{S'} \in \cH$ with
    \begin{align*}
        h_{S'}(x_i) \geq w_i + \eps' \quad \text{ for } \quad x_i \in S', \\
        h_{S'}(x_i) \leq w_i - \eps' \quad \text{ for } \quad x_i \not\in S'.
    \end{align*}
    The fat-shattering dimension of $\cH$ at scale $\eps'$, denoted by $\mathrm{fat}_{\eps'}(\cH)$, is the size of the largest set $S$ that is $\eps'$-shattered by $\cH$. 
\end{definition}

\begin{definition}[Pseudo-dimension]
    Suppose $\cH$ is a class of real-valued hypotheses. The pseudo-dimension of $\cH$, denoted by $\mathrm{Pdim}(\cH)$, is defined as $\lim_{\eps' \downarrow 0} \mathrm{fat}_{\eps'}(\cH)$.
\end{definition}


\subsection{Definitions with Randomized Predictors}
We define the variants of step calibration and multi-accuracy that allow for randomized predictors $\mathbf{p}^* \in \Delta(\cP)$. 
These definitions follow straightforwardly from Definitions\ref{def:detstep} and \ref{def:detma}, by taking an additional expectation over the draw of $p^* \sim \mathbf{p}^*$.

\begin{definition}[Randomized $(\cG, \cH, \eps)$-Step Calibration] \label{def:rand-step}
    A randomized predictor $\bp^* \in \Delta(\cP)$ is $(\cG, \cH, \eps)$-step calibrated if for all $v \in [0, 1]$, $w \in [0, 1]$, $h \in \cH$, and $g \in \cG$, 
    \begin{align*}
        \abs{\EEsc{(x, y) \sim D, \, p^* \sim \bp^*}{(y - p^*(x)) \cdot \mathbbm{1}[p^*(x) \leq v, h(x) \leq w]}{g(x) =1}} \leq \eps \cdot \sqrt{P_g^{-1}}.
    \end{align*}
\end{definition} 

\begin{definition}[Randomized $(\cG, \cH, \eps)$-Multiaccuracy] \label{def:rand-multi}
    A randomized predictor $\bp^* \in \Delta(\cP)$ is $(\cG, \cH, \eps)$-multiaccurate if for all $w \in [0, 1]$, $h \in \cH$, and $g \in \cG$,
    \begin{align*}
        \abs{\EEsc{(x, y) \sim D,\, p^* \sim \bp^*}{(y - p^*(x)) \cdot \mathbbm{1}[h(x) \leq w]}{g(x)=1}} \leq \eps \cdot \sqrt{P_g^{-1}}.
    \end{align*}
\end{definition}

\section{Deferred Results and Proofs from Section~\ref{sec:pan}} \label{sec:pan-proofs}
\subsection{Loss Outcome Indistinguishability}
For completeness, we show that the two characterizations of Decision and Hypothesis OI that we gave in Section~\ref{sec:pan} are equivalent. 
This derivation is a  straightforward extension of the machinery that Gopalan et al. \cite{gopalan_loss_2022} developed for the omniprediction setting.

Suppose the predictor $p^*$ is deterministic. 
We can deconstruct the Decision OI condition as follows. 
\begin{align*}
    &\bigg |{ \EEsc{(x, y) \sim D}{\ell(k_\ell (p^*(x)), y) }{g(x)=1} - \EEsc{\substack{x \sim D \\ \tilde{y} \sim \mathrm{Ber}(p^*(x))}}{\ell(k_\ell (p^*(x)), \tilde{y}) }{g(x)=1}} \bigg | \\
    &\quad = \bigg| \EEsc{(x, y) \sim D}{y \cdot \ell(k_\ell (p^*(x)), 1) + (1-y) \cdot \ell(k_\ell (p^*(x)), 0)}{g(x)=1} \\
    &\qquad \quad \qquad - \EEsc{\substack{x \sim D \\ \tilde{y} \sim \mathrm{Ber}(p^*(x))}}{\tilde{y} \cdot \ell(k_\ell (p^*(x)), 1) + (1-\tilde{y}) \cdot \ell(k_\ell (p^*(x)), 0)}{g(x)=1} \bigg| \\
    &\quad = \bigg| \EEsc{(x, y) \sim D}{y \cdot \Delta \ell(k_\ell (p^*(x))) + \ell(k_\ell (p^*(x)), 0)}{g(x)=1} \\
    &\qquad \quad \qquad - \EEsc{x \sim D}{\EEsc{\tilde{y} \sim \mathrm{Ber}(p^*(x))}{\tilde{y}}{x} \cdot \Delta \ell(k_\ell (p^*(x))) + \ell(k_\ell (p^*(x)), 0)}{g(x)=1} \bigg| \\
    &\quad = \bigg| \EEsc{(x, y) \sim D}{(y - p^*(x)) \cdot \Delta \ell(k_\ell (p^*(x))) }{g(x)=1} \bigg|.
\end{align*}

An identical argument applies to Hypothesis OI.
\begin{align*}
    &\bigg |{ \EEsc{(x, y) \sim D}{\ell(h(x), y) }{g(x)=1} - \EEsc{\substack{x \sim D \\ \tilde{y} \sim \mathrm{Ber}(p^*(x))}}{\ell(h(x), \tilde{y}) }{g(x)=1}} \bigg | \\
    &\quad = \bigg| \EEsc{(x, y) \sim D}{y \cdot \ell(h(x), 1) + (1-y) \cdot \ell(h(x), 0)}{g(x)=1} \\
    &\qquad \quad \qquad - \EEsc{\substack{x \sim D \\ \tilde{y} \sim \mathrm{Ber}(p^*(x))}}{\tilde{y} \cdot \ell(h(x), 1) + (1-\tilde{y}) \cdot \ell(h(x), 0)}{g(x)=1} \bigg| \\
    &\quad = \bigg| \EEsc{(x, y) \sim D}{y \cdot \Delta \ell(h(x)) + \ell(h(x), 0)}{g(x)=1} \\
    &\qquad \quad \qquad - \EEsc{x \sim D}{\EEsc{\tilde{y} \sim \mathrm{Ber}(p^*(x))}{\tilde{y}}{x} \cdot \Delta \ell(h(x)) + \ell(h(x), 0)}{g(x)=1} \bigg| \\
    &\quad = \bigg| \EEsc{(x, y) \sim D}{(y - p^*(x)) \cdot \Delta \ell(h(x)) }{g(x)=1} \bigg|.
\end{align*}

The same steps apply to randomized $\bp^*$, with an additional expectation taken over the draw of $p^* \sim \bp^*$. 
For Decision OI, this means
\begin{align*}
    &\bigg |{ \EEsc{(x, y) \sim D, \, p^* \sim \bp^* }{\ell(k_\ell (p^*(x)), y) }{g(x)=1} - \EEsc{\substack{x \sim D, \, p^* \sim \bp^* \\ \tilde{y} \sim \mathrm{Ber}(p^*(x))}}{\ell(k_\ell (p^*(x)), \tilde{y}) }{g(x)=1}} \bigg | \\
    &\quad = \bigg| \EEsc{(x, y) \sim D, \, p^* \sim \bp^*}{y \cdot \ell(k_\ell (p^*(x)), 1) + (1-y) \cdot \ell(k_\ell (p^*(x)), 0)}{g(x)=1} \\
    &\qquad \qquad - \EEsc{\substack{x \sim D, \, p^* \sim \bp^*\\ \tilde{y} \sim \mathrm{Ber}(p^*(x))}}{\tilde{y} \cdot \ell(k_\ell (p^*(x)), 1) + (1-\tilde{y}) \cdot \ell(k_\ell (p^*(x)), 0)}{g(x)=1} \bigg| \\
    &\quad = \bigg| \EEsc{(x, y) \sim D, \, p^* \sim \bp^*}{y \cdot \Delta \ell(k_\ell (p^*(x))) + \ell(k_\ell (p^*(x)), 0)}{g(x)=1} \\
    &\qquad \qquad - \EEsc{x \sim D, \, p^* \sim \bp^*}{\EEsc{\tilde{y} \sim \mathrm{Ber}(p^*(x))}{\tilde{y}}{x, \, p^*} \cdot \Delta \ell(k_\ell (p^*(x))) + \ell(k_\ell (p^*(x)), 0)}{g(x)=1} \bigg| \\
    &\quad = \bigg| \EEsc{(x, y) \sim D, \, p^* \sim \bp^*}{(y - p^*(x)) \cdot \Delta \ell(k_\ell (p^*(x))) }{g(x)=1} \bigg|.
\end{align*}

For Hypothesis OI, this means
\begin{align*}
    &\bigg |{ \EEsc{(x, y) \sim D, \, p^* \sim \bp^*}{\ell(h(x), y) }{g(x)=1} - \EEsc{\substack{x \sim D, \, p^* \sim \bp^* \\ \tilde{y} \sim \mathrm{Ber}(p^*(x))}}{\ell(h(x), \tilde{y}) }{g(x)=1}} \bigg | \\
    &\quad = \bigg| \EEsc{(x, y) \sim D, \, p^* \sim \bp^*}{y \cdot \ell(h(x), 1) + (1-y) \cdot \ell(h(x), 0)}{g(x)=1} \\
    &\qquad \quad \qquad - \EEsc{\substack{x \sim D, \, p^* \sim \bp^* \\ \tilde{y} \sim \mathrm{Ber}(p^*(x))}}{\tilde{y} \cdot \ell(h(x), 1) + (1-\tilde{y}) \cdot \ell(h(x) 0)}{g(x)=1} \bigg| \\
    &\quad = \bigg| \EEsc{(x, y) \sim D, \, p^* \sim \bp^*}{y \cdot \Delta \ell(h(x)) + \ell(h(x), 0)}{g(x)=1} \\
    &\qquad \quad \qquad - \EEsc{x \sim D, \, p^* \sim \bp^*}{\EEsc{\tilde{y} \sim \mathrm{Ber}(p^*(x))}{\tilde{y}}{x, \, p^*} \cdot \Delta \ell(h(x)) + \ell(h(x), 0)}{g(x)=1} \bigg| \\
    &\quad = \bigg| \EEsc{(x, y) \sim D, \, p^* \sim \bp^*}{(y - p^*(x)) \cdot \Delta \ell(h(x)) }{g(x)=1} \bigg|.
\end{align*}

\subsection{From Deterministic Step Calibration to Deterministic Panprediction}
First, we restate and prove Lemma~\ref{lem:decomp}, the decomposition helper lemma used to prove Theorem~\ref{thm:detpan}.
\decomp*

\begin{proof}[Proof of Lemma~\ref{lem:decomp}]
     For the first statement, fix any $h \in \cH$ and $w = 1$. Recall that by definition of $(\cG, \cH, \eps)$-step calibration, we have that for all $v \in [0, 1]$ and $g \in \cG$,
    \begin{align*}
        &\abs{\EEsc{(x, y) \sim D}{(y - p^*(x)) \cdot \mathbbm{1}[p^*(x) \leq v, h(x) \leq 1]}{g(x) = 1}} \\
        &\qquad = \abs{\EEsc{(x, y) \sim D}{(y - p^*(x)) \cdot \mathbbm{1}[p^*(x) \leq v]}{g(x) = 1}} \leq \eps \cdot \sqrt{P_g^{-1}}.
    \end{align*}
    Hence, $p^*$ is $(\cG, \emptyset, \eps)$-step calibrated.
    
    For the second statement, fix $v = 1$. $(\cG, \cH, \eps)$-step calibration also guarantees that for all $h \in \cH$, $w \in [0, 1]$, and $g \in \cG$,
    \begin{align*}
        &\abs{\EEsc{(x, y) \sim D}{(y - p^*(x)) \cdot \mathbbm{1}[p^*(x) \leq 1, h(x) \leq w]}{g(x) = 1}} \\
        &\qquad = \abs{ \EEsc{(x, y) \sim D}{(y - p^*(x)) \cdot \mathbbm{1}[h(x) \leq w]}{g(x) = 1} } \leq \eps \cdot \sqrt{P_g^{-1}}.
    \end{align*}
    Hence, $p^*$ is $(\cG, \cH, \eps)$-multiaccurate.
\end{proof}

Next, we prove Lemma~\ref{lemma:approx}, the approximation-theoretic helper lemma used to prove Theorem~\ref{thm:detpan}. 
The following exposition closely follows that of Okoroafor, Kleinberg, and Kim \cite{okoroafor_near-optimal_2025}. 

We begin by recalling Assumption~\ref{as:bv}.

\begin{definition}[Bounded variation loss]
    The total variation of a loss function $\ell: \hat{\cY} \times \cY \to [-1, 1]$ in its first argument is
    \begin{align*}
        \sup_{y \in \cY} V(\ell(\cdot, y)) = \sup_{y \in \cY} \, \sup_{m \in \mathbb{N}} \, \sup_{0 = z_0 < \dots < z_m = 1} \sum_{j=1}^{m} \abs{\ell(z_j, y) - \ell(z_{j-1}, y)}.
    \end{align*}
    The class of all loss functions whose total variation in the first argument is bounded is denoted by $\cL_{\mathrm{BV}}$.
\end{definition}

For simplicity, we can say that $\cL_{\mathrm{BV}} = \{\ell: [0, 1] \to [-1, 1] \mid \sup_{y} V(\ell(\cdot, y)) \leq 1\}$. 
This is without loss of generality, since $\ell$ with total variation $v < \infty$ can be rescaled to $v^{-1} \cdot \ell$ to fall into this simpler class. 
Rescaling does not change relevant properties of the loss, e.g., what predictor minimizes the loss.
Observe that $\Delta \cL_{\mathrm{BV}} = \{ \Delta \ell \mid \ell \in \cL_{\mathrm{BV}}, \, V(\Delta \ell) \leq 2 \}$ is also a class of bounded variation functions. 

An important subset of bounded variation losses is proper losses. 

\begin{definition}[Proper loss]
    A bounded loss function $\ell: [0, 1] \times \cY \to [-1, 1]$ is a proper loss if for all $p, q \in [0, 1]$,
    \begin{align*}
        \EEs{y \sim \mathrm{Ber}(p)}{\ell(p, y)} \leq \EEs{y \sim \mathrm{Ber}(p)}{\ell(q, y)}. 
    \end{align*}
\end{definition}

A proper loss $\ell$ has bounded variation, as $\Delta \ell(p) = \ell(p, 1) - \ell(p, 0)$ decreases monotonically as $p \to 1$. 
Additionally, when $\ell$ is bounded on $[-1, 1]$, it follows that $V(\Delta \ell) \leq 2$.   

Given a loss class $\cL \subseteq \cL_{\mathrm{BV}}$, our goal is to approximate functions in $\Delta \cL \subseteq \Delta \cL_{\mathrm{BV}}$ with a finite basis of simple functions. 
To begin, we rigorously define what we consider an approximate basis, then give two examples. 
At a technical level, our approach combines those of Okoroafor, Kleinberg, and Kim \cite{okoroafor_near-optimal_2025} and Qiao and Zhao \cite{qiao_truthfulness_2025}.

\begin{definition}[$\eps$-approximate basis]
\label{def:approximatebasis}
    Let $\cF \subseteq \{ f: [0, 1] \to [-1, 1] \}$ be a function class. A set of functions $\cB = \{ \beta: [0, 1] \to [-1, 1] \}$ forms a $\eps$-approximate basis with sparsity $s \in \mathbb{N}$ and coefficient norm $C > 0$ if for all $f \in \cF$, there exists a finite subset $\beta_1, \dots, \beta_s \in \cB$ and coefficients $c_1, \dots, c_s \in [-1, 1]$ such that for all $p \in [0, 1]$,
    \begin{align*}
        \abs{ f(x) - \sum_{j=1}^{s} c_j \cdot \beta_j(x) } \leq \eps \quad \text{and} \quad \sum_{j=1}^{s} \abs{c_j} \leq C.
    \end{align*}
\end{definition}

\begin{definition}[Threshold functions]
    For $v \in [0, 1]$, define $\mathrm{Th}_v:[0, 1] \to \{-1, 1\}$, 
    \begin{align*}
        \mathrm{Th}_v(p) = \mathrm{sign}(v - p) = 
        \begin{cases}
            +1 \quad \text{if} \quad p \leq v \\
            -1 \quad \text{if} \quad p > v.
        \end{cases}
    \end{align*}
\end{definition}

\begin{definition}[Step functions]
    For $v \in [0, 1]$, define $\text{Step}_v: [0, 1] \to \{0, 1\}$, $\mathrm{Step}_v(p) = \mathbbm{1}[p \leq v]$.
\end{definition}

As observed in Lemma A.1 of Qiao and Zhao \cite{qiao_truthfulness_2025}, threshold functions and step functions are related by the fact that
\begin{equation}\label{eq:step}
    \mathrm{Th}_v(p) = \mathbbm{1}[p \leq v] - \mathbbm{1}[p > v] = 2 \cdot \mathbbm{1}[p \leq v] - \mathbbm{1}[p \leq 1].
\end{equation}

\begin{proposition}[Lemma 5.6 in \cite{okoroafor_near-optimal_2025}]
    For any $\eps > 0$, the uncountably infinite set of threshold functions
    \begin{align*}
        \mathbf{Th} \coloneq \left\{ \mathrm{Th}_{v} \mid v \in [0, 1] \right\}
    \end{align*}
    forms a $\eps$-basis for $\Delta \cL_\mathrm{BV} = \{ \Delta \ell \mid \ell \in \cL_{\mathrm{BV}} \}$ with sparsity $\lceil \frac{2}{\eps}+1 \rceil$ and coefficient norm $3$. 
    \label{prop:bv}
\end{proposition}

The fact that the approximate basis of $\Delta \cL_{\mathrm{BV}}$ is uncountably infinite is not a problem for the proof of Lemma~\ref{lemma:approx}, but can be problematic for the design of algorithms that take the approximate basis functions as input. Fortunately, an approximate basis of finite size can always be found when the inputs to $\Delta \ell$, i.e., the predictions made by $p^* \in \cP$ and $h \in \cH$, are quantized to the $\lambda$-net of the unit interval, $I_\lambda$. This observation is formalized below. 

\begin{proposition}[Lemma 5.7 in \cite{okoroafor_near-optimal_2025}]
  Suppose the inputs to each $\Delta \ell \in \Delta \cL_{\mathrm{BV}}$ is restricted to $I_\lambda$. Then for all $\eps > 0$, the finite set of threshold functions
    \begin{align*}
        \mathbf{Th}_{\lambda} \coloneq \left\{ \mathrm{Th}_{v} \mid v \in I_\lambda \right\}
    \end{align*}
    forms an $\eps$-basis for $\Delta \cL_\mathrm{BV}$ with sparsity $\lceil \frac{2}{\eps}+1 \rceil$ and coefficient norm $3$. 
    \label{rem:th}
\end{proposition}

With the above results in hand, we can prove Lemma~\ref{lemma:approx}.
\approx*

\begin{proof}[Proof of Lemma~\ref{lemma:approx}]
    By Definition~\ref{def:approximatebasis}, Proposition~\ref{prop:bv}, and the triangle inequality, we have that
    \begin{align*}
        &\abs{ \EEsc{(x, y) \sim D}{(y-p^*(x)) \cdot \Delta \ell(f(x))}{g(x)=1}} \\
        &\qquad \leq \abs{ \EEsc{(x, y) \sim D}{(y-p^*(x)) \cdot \sum_{j=1}^{s} c_j \cdot \mathrm{Th}_{v_j}(f(x)) }{g(x)=1}} \\
        &\qquad \qquad + \abs{ \EEsc{(x, y) \sim D}{(y-p^*(x)) \cdot \left( \Delta \ell(f(x)) - \sum_{j=1}^{s} c_j \cdot \mathrm{Th}_{v_j}(f(x)) \right)}{g(x)=1}} \\
        &\qquad \leq \abs{ \sum_{j=1}^{s} c_j \cdot \EEsc{(x, y) \sim D}{ (y-p^*(x)) \cdot \mathrm{Th}_{v_j}(f(x)) }{g(x)=1}} + \eps.
    \end{align*}
    The above line can be further upper bounded as 
    \begin{align*}
        &\sum_{j=1}^{s} \abs{c_j} \cdot \abs{ \EEsc{(x, y) \sim D}{(y-p^*(x)) \cdot \mathrm{Th}_{v_j}(f(x)) }{g(x)=1}} + \eps \\
        &\qquad \leq \left( \sum_{j=1}^{s} \abs{c_j} \right) \cdot \max_{j \in [s]} \abs{ \EEsc{(x, y) \sim D}{(y-p^*(x)) \cdot \mathrm{Th}_{v_j}(f(x)) }{g(x)=1}} + \eps \\
        &\qquad \leq 3 \sup_{v \in [0, 1]} \abs{ \EEsc{(x, y) \sim D}{(y-p^*(x)) \cdot \mathrm{Th}_{v}(f(x)) }{g(x)=1}} + \eps,
    \end{align*}
    where the second inequality follows from recalling that threshold functions form an approximate basis with coefficient norm 3. Recalling Equation~\ref{eq:step} and invoking the triangle inequality, we can rewrite the above line as
    \begin{align*}
        &3 \sup_{v \in [0, 1]} \abs{ \EEsc{(x, y) \sim D}{(y-p^*(x)) \cdot \left( 2 \cdot \mathbbm{1}[f(x) \leq v] - \mathbbm{1}[f(x) \leq 1] \right) }{g(x)=1}} + \eps\\
        &\qquad \leq 9 \sup_{v \in [0, 1]} \abs{ \EEsc{(x, y) \sim D}{(y-p^*(x)) \cdot \mathbbm{1}[f(x) \leq v]}{g(x)=1} } + \eps \\
        &\qquad \leq 9 \sup_{v \in [0, 1]} \abs{ \EEsc{(x, y) \sim D}{(y-p^*(x)) \cdot \mathbbm{1}[f(x) \leq v]}{g(x)=1} } + \eps \cdot \sqrt{P_g^{-1}}.
    \end{align*}
    The final inequality is trivial, since $P_g$ is positive and at most 1.
\end{proof}

\subsection{From Randomized Step Calibration to Randomized Panprediction}
We show that the same loss outcome indistinguishability approach used above can be used to prove that randomized step calibration implies randomized panprediction. 

\begin{theorem}\label{thm:randpan}
    If a randomized predictor $\bp^* \in \Delta(\cP)$ is $(\cG, \cH, \eps)$-step calibrated, then $\bp^*$ is a randomized $(\cL, \cG, \cH, O(\eps))$-panpredictor.
\end{theorem}

As seen before, we prove Theorem~\ref{thm:randpan} by decomposing the step calibration guarantee into a calibration and a multiaccuracy guarantee (Lemma~\ref{lemma:decomp2}), then showing that the calibration guarantee implies Decision OI (Lemma~\ref{lemma:decoi2}) and that the multiaccuracy guarantee implies Hypothesis OI (Lemma~\ref{lemma:hypoi2}). 
\begin{lemma}
     If a randomized predictor $\bp^* \in \Delta(\cP)$ is $(\cG, \cH, \eps)$-step calibrated, then $\bp^*$ is $(\cG, \emptyset, \eps)$-step calibrated and $(\cG, \cH, \eps)$-multiaccurate. 
\label{lemma:decomp2}
\end{lemma}
\begin{proof}[Proof of Lemma~\ref{lemma:decomp2}]
    For the first statement, fix any $h \in \cH$ and $w = 1$. Recall that from the definition of  $(\cG, \cH, \eps)$-step calibration, we have that for all $v \in [0, 1]$ and $g \in \cG$,
    \begin{align*}
        &\abs{\EEsc{(x, y) \sim D, \, p^* \sim \bp^*}{(y - p^*(x)) \cdot \mathbbm{1}[p^*(x) \leq v, h(x) \leq 1]}{g(x) = 1}} \\
        &\qquad = \abs{\EEsc{(x, y) \sim D, \, p^* \sim \bp^*}{(y - p^*(x)) \cdot \mathbbm{1}[p^*(x) \leq v]}{g(x) = 1}} \leq \eps \cdot \sqrt{P_g^{-1}}.
    \end{align*}
    Hence, $p^*$ is $(\cG, \emptyset, \eps)$-step calibrated.
    
    For the second statement, fix $v = 1$.  $(\cG, \cH, \eps)$-step calibration also guarantees that for all $h \in \cH$, $w \in [0, 1]$, and $g \in \cG$,
    \begin{align*}
        &\abs{\EEsc{(x, y) \sim D, \, p^* \sim \bp^*}{(y - p^*(x)) \cdot \mathbbm{1}[p^*(x) \leq 1, h(x) \leq w]}{g(x) = 1}} \\
        &\qquad = \abs{ \EEsc{(x, y) \sim D, \, p^* \sim \bp^*}{(y - p^*(x)) \cdot \mathbbm{1}[h(x) \leq w]}{g(x) = 1} } \leq \eps \cdot \sqrt{P_g^{-1}}.
    \end{align*}
    Hence, $p^*$ is $(\cG, \cH, \eps)$-multiaccurate.
\end{proof}

\begin{lemma}
    If a randomized predictor $\bp^* \in \Delta(\cP)$ is $(\cG, \emptyset, \eps)$-step calibrated, then for all $g \in \cG$,
    \begin{align*}
        \abs{ \EEsc{(x, y) \sim D, \, p^* \sim \bp^*}{(y - p^*(x)) \cdot \Delta \ell(k_\ell(p^*(x)))}{g(x)=1} } \leq O(\eps) \cdot \sqrt{P_g^{-1}}.
    \end{align*}
    \vspace{-\baselineskip}
    \label{lemma:decoi2}
\end{lemma}

\begin{proof}[Proof of Lemma~\ref{lemma:decoi2}]
    Since for any $\ell: [0, 1] \times \cY \to [-1, 1]$ and $p, q \in [0, 1]$, by the definition of the function $k_\ell$, 
    \begin{align*}
        \EEs{y \sim \mathrm{Ber}(p)}{\ell(k_{\ell}(p), y)} \leq \EEs{y \sim \mathrm{Ber}(p)}{\ell(k_{\ell}(q), y)},
    \end{align*}
    the map $\ell(k_\ell(\cdot), y): [0, 1] \to [-1, 1]$ is a proper scoring rule. Hence this map is contained in the bounded variation loss class $\cL_{\mathrm{BV}}$. Then by Lemma~\ref{lemma:approx},
    \begin{align*}
        &\abs{ \EEsc{(x, y) \sim D}{(y - p^*(x)) \cdot \Delta \ell(k_\ell(p^*(x)))}{g(x)=1} } \\
        &\qquad \leq 9 \sup_{v \in [0, 1]} \abs{ \EEsc{(x, y) \sim D, \, p^* \sim \bp^*}{(y - p^*(x)) \cdot \mathbbm{1}[p^*(x) \leq v]}{g(x) = 1} } + \eps \cdot \sqrt{P_g^{-1}} \\
        &\qquad \leq O(\eps) \cdot \sqrt{P_g^{-1}},
    \end{align*}
    where the last inequality follows from the $(\cG, \emptyset, \eps)$-step calibration guarantee. 
\end{proof}

\begin{lemma}
    If a randomized predictor $\bp^* \in \Delta(\cP)$ is $(\cG, \cH, \eps)$-multiaccurate, then for all $h \in \cH$ and $g \in \cG$,
    \begin{align*}
        \abs{ \EEsc{(x, y) \sim D, \, p^* \sim \bp^*}{(y - p^*(x)) \cdot \Delta \ell(h(x))}{g(x)=1} } \leq O(\eps) \cdot \sqrt{P_g^{-1}}.
    \end{align*}
    \vspace{-\baselineskip}
    \label{lemma:hypoi2}
\end{lemma}

\begin{proof}[Proof of Lemma~\ref{lemma:hypoi2}]
     By assumption, $\ell \in \cL_{\mathrm{BV}}$. Then by applying Lemma~\ref{lemma:approx} with any $h \in \cH$, $h: \cX \to [0, 1]$,
     \begin{align*}
        &\abs{ \EEsc{(x, y) \sim D, \, p^* \sim \bp^*}{(y - p^*(x)) \cdot \Delta \ell(h(x))}{g(x)=1} } \\
        &\qquad \leq 9 \sup_{w \in [0, 1]} \abs{ \EEsc{(x, y) \sim D, \, p^* \sim \bp^*}{(y - p^*(x)) \cdot \mathbbm{1}[h(x) \leq w]}{g(x) = 1} } + \eps \cdot \sqrt{P_g^{-1}} \\
        &\qquad \leq O(\eps) \cdot \sqrt{P_g^{-1}},
    \end{align*}
    where the last inequality follows from the $(\cG, \cH, \eps)$-multiaccuracy guarantee.
\end{proof}

\section{Deferred Results and Proofs from Section~\ref{sec:algo}} \label{sec:algo-proofs}
In this section, we state and prove guarantees about our  deterministic and randomized step calibration algorithms. 

\subsection{Background}
We begin by collecting notation and useful results from online learning, adaptive data analysis, and multi-objective learning. 
This exposition closely follows that of Haghtalab, Jordan, and Zhao  \cite{DBLP:conf/nips/HaghtalabJ023}.

\paragraph{Online learning}
Online learning models a $T$-round interaction between a learner and an adversary. On each round $t \in [T]$, the learner chooses an action $a^{(t)}$ from an action set $\cA$ and the adversary chooses a bounded cost function $c^{(t)}: \cA \to [0, 1]$. In the sequel, we focus on stochastic cost functions $c^{(t)}: \cA \times (\cX \times \cY) \to [0, 1]$ that take both the learner's action and a datapoint $(x, y) \sim D$ as inputs. It is natural to minimize expected cost $\cR_{c^{(t)}}(\cdot) = \EEs{(x, y) \sim D}{c^{(t)}(\cdot, (x, y))}: \cA \to [0, 1]$. When the choices of action and cost function are non-deterministic, that is, $p \in \Delta(\cA)$ and $q \in \Delta(\cC)$, we can also consider $\cR_{q}(p) = \EEs{a \sim p, \, c \sim q}{\cR_{c}(a)} = \EEs{a \sim p, \, c \sim q}{\EEs{(x, y) \sim D}{c(a, (x, y)}}$. Importantly, the stochastic cost functions we consider have linear structure in the action, that is, for all $(x, y) \in \cX \times \cY$, $c(\cdot, (x, y)): \cA \to [0, 1]$ is linear. 

The learner's regret (to the best fixed action) is defined as 
\begin{align*}
    \mathrm{Reg}(a^{(1:T)}, c^{(1:T)}) = \sum_{t=1}^{T} c^{(t)}(a^{(t)}) - \min_{a^* \in \cA} \sum_{t=1}^{T} c^{(t)}(a^*). 
\end{align*}
When working with stochastic cost functions, we have that
\begin{align*}
    &\mathrm{Reg}(a^{(1:T)}, \{ \cR_{c^{(t)}}(\cdot) \}^{(1:T)}) \\
    &\qquad \qquad = \sum_{t=1}^{T} \EEs{(x, y) \sim D}{c^{(t)}(a^{(t)}, (x, y))} - \min_{a^* \in \cA} \sum_{t=1}^{T} \EEs{(x, y) \sim D}{c^{(t)}(a^*, (x, y))}. 
\end{align*}

The learner's weak regret to a minimax benchmark $T \cdot \min_{a^* \in \cA} \max_{c^* \in \cC} c^*(a^*)$ is given by
\begin{align*}
    \mathrm{Reg}_{\mathrm{weak}}(a^{(1:T)}, c^{(1:T)}) = \sum_{t=1}^{T} c^{(t)}(a^{(t)}) - T \cdot \min_{a^* \in \cA} \max_{c^* \in \cC} c^*(a^*). 
\end{align*}
Identically, when working with stochastic cost functions, we have that
\begin{align*}
    &\mathrm{Reg}_{\mathrm{weak}}(a^{(1:T)}, \{ \cR_{c^{(t)}}(\cdot) \}^{(1:T)}) \\
    &\qquad \qquad = \sum_{t=1}^{T} \EEs{(x, y) \sim D}{c^{(t)}(a^{(t)}, (x, y))} - T \cdot \min_{a^* \in \cA} \max_{c^* \in \cC} \EEs{(x, y) \sim D}{c^*(a^*, (x, y))}. 
\end{align*}

The standard Hedge algorithm \cite{freund_decision-theoretic_1997} obtains a sublinear regret bound. 
\begin{proposition}[Hedge \cite{freund_decision-theoretic_1997}]
     If the action sequence in the $k$-simplex $a^{(1:T)} \in \Delta_k$ is selected by the Hedge algorithm, then for any adversarial choice of cost functions $c^{(1:T)}$,
    \begin{align*}
        \mathrm{Reg}(a^{(1:T)}, c^{1:T)}) \leq C \sqrt{\log(k) T},
    \end{align*}
    where $C > 0$ is a universal constant.
\label{prop:hedge}
\end{proposition}

Given sampling access to distribution $D$ and an action $a \in \cA$, a natural way to estimate the expected cost $\cR_{c}(a)$ is sampling $(x, y) \sim D$ and evaluating $c(a, (x, y))$. When $c$ has linear structure, this approach incurs sublinear estimation error with high probability. 
\begin{proposition}[Stochastic approximation \cite{nemirovski_robust_2009}]
    Let $(x^{(1)}, y^{(1)}), (x^{(2)}, y^{(2)}), \dots, (x^{(T)}, y^{(T)}) \overset{\mathrm{iid}}{\sim} D$. Suppose that all $c \in \cC$ are linear and that at each round $t \in [T]$, after picking action $a^{(t)}$ with an online learning algorithm, the expected cost $\cR_{c^{(t)}}(a^{(t)})$ is estimated with $\hat{c}^{(t)}(a) \coloneq c^{(t)}(a, (x^{(t)}, y^{(t)}))$. Then with probability at least $1-\delta$,
    \begin{align*}
        \abs{ \mathrm{Reg}(a^{(1:T)}, \{ \cR_{c^{(t)}}(\cdot)\}^{(1:T)} ) - \mathrm{Reg}(a^{(1:T)}, \hat{c}^{(1:T)}) } \leq C \sqrt{T \log(1/\delta)},
    \end{align*}
    where $C > 0$ is a universal constant.
\label{prop:stoch}
\end{proposition}

We sometimes consider best-response oracles of the following form. 

\begin{definition}[Best-response oracle]
    A best response oracle takes an action $a \in \cA$, a set of stochastic cost functions $\cC$, and error tolerance $\eps$, then returns a cost function $c \in \cC$ such that
    \begin{align*}
        \EEs{(x, y) \sim D}{c(a, (x, y))} \leq \min_{c^* \in \cC} \EEs{(x, y) \sim D}{c^*(a, (x, y))} + \eps.
    \end{align*}
    Such a choice of $c \in \cC$ is referred to as a $\eps$-best response to the action $a$.
\end{definition}

\paragraph{Adaptive data analysis}
Adaptive data analysis models a $T$-round interaction between a data analyst and a mechanism. At each round $t \in [T]$, the analyst poses a query about a fixed distribution $D$. Using only samples $z^{(1:n)} = \{ (x_1, y_1), \dots, (x_n, y_n) \} \overset{\mathrm{iid}}{\sim} D$, the mechanism must sequentially $T$ queries (where each query may adversarially depend on all previous queries and answers) up to small additive error with high probability. 

We focus on minimization queries over finite sets, which take the following form. Suppose a family of loss functions $\{ \ell \}$ is parameterized by a finite set of parameters $\Theta$. That is, for each loss function $\ell \in \cL$, $\ell: (\cX \times \cY)^n \times \Theta \to [0, 1]$. Further suppose that each loss has $\Delta$-sensitivity. That is, for all ``datasets'' $z^{(1:n)}$ and $\bar{z}^{(1:n)}$ that differ in just one data point, 
\begin{align*}
    \sup_{\theta \in \Theta} \abs{ \ell(z^{(1:n)}, \theta) - \ell(\bar{z}^{(1:n)}, \theta) } \leq \Delta.
\end{align*}
Fix $\alpha, \beta \in (0, 1)$. The $(\alpha, \beta)$-minimization query associated with loss $\ell$ seeks $\theta \in \Theta$ such that with probability $1-\beta$,
\begin{align*}
    \EEs{z^{(1:n)} \sim D^n}{\ell(z^{(1:n)}, \theta)} \leq \min_{\theta^* \in \Theta} \EEs{z^{(1:n)} \sim D^n}{\ell(z^{(1:n)}, \theta^*)} + \alpha.
\end{align*}

Minimization queries can be naturally used to best respond to stochastic cost functions discussed above, as long as the cost functions are parameterized by a finite set $\Theta$. Importantly, adaptive data analysis allows cost functions to be chosen adaptively. 
\begin{proposition}[Adaptive minimization queries \cite{bassily_algorithmic_2015}]
    There exists an algorithm that, with probability at least $1-\delta$, can $\eps$-best respond to an adaptive sequence of $T$ stochastic costs that each have sensitivity $1/n$ and are bounded in $[0, 1]$, using at most
    \begin{align*}
        n = O\left( \frac{\sqrt{T} \cdot \log\left( |\Theta| / \eps \right) \cdot \log^{3/2}(1/\eps\delta)}{\eps^2} \right)
    \end{align*}
    samples from $\cD$.
\label{prop:adaptive}
\end{proposition}
The algorithm that witnesses this sample complexity is an instantiation of the exponential mechanism, and is detailed in prior works \cite{mcsherry_mechanism_2007, bassily_algorithmic_2015}.

\subsection{Multi-objective Learning}
To begin, recall the definition of multi-objective learning. 
\begin{definition}[\cite{DBLP:conf/nips/HaghtalabJ023}]
    A multi-objective learning problem is defined by a distribution $D$, a hypothesis class $\cF \subseteq \{ f: \cX \to [0, 1]\}$, and a set of objectives $\cO = \{ \ell: \cF \times (\cX \times \cY) \to [a, b] \}$. The goal of multi-objective learning is to find a hypothesis $\mathbf{f}^* \in \Delta(\cF)$ such that
    \begin{align*}
       \max_{\ell \in \cO} \EEs{(x, y) \sim D, \, f^* \sim \mathbf{f}^*}{\ell(f^*(x), y)} \leq \min_{f \in \cF} \max_{\ell \in \cO} \EEs{(x, y) \sim D}{\ell(f(x), y)} + \eps.
    \end{align*}
    If $\mathbf{f}^*$ is a point mass on a hypothesis $f^* \in \cH$, then $f^*$ is called a deterministic $\eps$-optimal solution to $(D, \cO, \cF)$. Otherwise, $\mathbf{f}^*$ is called a randomized $\eps$-optimal solution to $(D, \cO, \cF)$.
\end{definition}

Deterministic solutions to multi-objective problems can be found via ``no-regret vs best-response'' dynamics that select a sequence of hypotheses with no-regret online learning algorithms and a sequence of objectives using best-response oracles. The following propositions formalize this result. 
\begin{proposition}[No-Regret vs Best-Response \cite{DBLP:conf/nips/HaghtalabJ023}]
    Consider a multi-objective learning problem $(D, \cO, \cH)$ where the learner chooses $p^{(1:T)} \in \Delta(\cH)$ and the adversary chooses $q^{(1:T)} \in \Delta(\cO)$. If the learner is no-regret, such that $\mathrm{Reg}_{\mathrm{weak}}(p^{(1:T)}, \{ \cR_{q^{(t)}}(\cdot) \}^{(1:T)})$, and the adversary $\eps$-best responds to the sequence of costs $\{ 1-\cR_{(\cdot)}(p^{(t)}) \}^{(1:T)}$ using $q^{(1:T)}$, then there is a timestep $t \in [T]$ where $p^{(t)}$ is a deterministic $2\eps$-optimal solution to $(D, \cO, \cH)$.
\label{prop:nrbr}
\end{proposition}

\begin{proposition}[Deterministic solution \cite{DBLP:conf/nips/HaghtalabJ023}]
    Suppose a sequence of hypotheses $p^{(1:T)}$ contains a deterministic $\eps$-optimal solution to $(D, \cO, \cH)$. Then a deterministic $2\eps$-optimal solution $p^{(t)} \in p^{(1:T)}$ can be identified using $O(\eps^{-2}\log(|\cO|T/\delta))$ samples. 
\label{prop:nrbr2}
\end{proposition}

Randomized solutions to multi-objective problems can be found via ``no-regret vs no-regret'' dynamics that uses no-regret online learning algorithms to select both the sequence of hypotheses and the sequence of objectives. The following proposition formalizes this result. 
\begin{proposition}[No-Regret vs No-Regret \cite{DBLP:conf/nips/HaghtalabJ023}]
    Consider a multi-objective learning problem $(D, \cO, \cH)$ where the learner chooses $p^{(1:T)} \in \Delta(\cH)$ and the adversary chooses $q^{(1:T)} \in \Delta(\cO)$. If both players are no-regret, such that $\mathrm{Reg}_{\mathrm{weak}}(p^{(1:T)}, \{ \cR_{q^{(t)}}(\cdot) \}^{(1:T)}) \leq T\eps$ and $\mathrm{Reg}(q^{(1:T)}, \{ 1- \cR_{(\cdot)}(p^{(t)}) \}^{(1:T)}) \leq T\eps$, then the non-deterministic hypothesis $\bar{p} = \mathrm{Uniform}(p^{(1:T)})$ is a $2\eps$-optimal solution to $(D, \cO, \cH)$.
\label{prop:nrnr}
\end{proposition}

\paragraph{Application to Step Calibration}
Theorem~\ref{thm:multiobj} shows that $(\cG, \cH, \eps)$-step calibration can be formulated as a $(D, \cO_{\mathrm{sc}}, \cP)$-multi-objective learning problem. 
It remains to address the fact that $\cO_\mathrm{sc}$ is an uncountably infinite set of losses parameterized by 
\begin{align*}
    (\sigma, v, w, h, g) \in \{ \pm 1 \} \times [0, 1] \times [0, 1] \times \cH \times \cG.
\end{align*}
For our algorithms to work, a finite cover of $\cO_\mathrm{sc}$, denoted by $\hat{O}_\mathrm{sc}$, must be found.
If $\cH$ and $\cG$ have finite cardinality, and real-valued hypotheses are quantized to $I_\lambda$ for $\lambda = \Theta(\eps)$, a (lossless) finite cover of size $O(\eps^{-2}|\cH||\cG|)$ can be constructed by choosing parameters
\begin{align*}
    (\sigma, v, w, h, g) \in \{ \pm 1 \} \times I_{\Theta(\eps)} \times I_{\Theta(\eps)}  \times \cH \times \cG.
\end{align*} 
In the more general case, where $\cH$ and $\cG$ have infinite cardinality but bounded combinatorial dimensions, a finite cover can still be constructed. Namely, assume that $\mathrm{Pdim}(\cH) \leq d_H$ and $\mathrm{VC}(\cG) \leq d_G$. Then the composed class $\mathrm{Step} \circ \cH$ has VC dimension at most $d_H$. Furthermore, the class of pointwise conjunctions between $\mathrm{Step} \circ \cH$ and $\cG$, denoted by $(\mathrm{Step} \circ \cH) \wedge \cG$, has VC dimension at most $d_H + d_G$. It is well-known that given $\tilde{O}(\eps^{-1}(d_H + d_G + \log(1/\delta'))$ samples from $D$, an $\eps$-cover of $(\mathrm{Step} \circ \cH) \wedge \cG$ in the $L_1(D)$ metric can be algorithmically constructed, with failure probability at most $\delta'$. In fact, this cover is constructed by combining an $\eps$-cover of $\mathrm{Step} \circ \cH$ (denoted by $(\mathrm{Step} \circ \cH)'$) and an $\eps$-cover of $\cG$ (denoted by $\cG'$). By standard learning-theoretic arguments, the combined cover has size at most $1/\eps^{O(d_H+d_G)}$ \cite{benedek_learnability_1991, dai_learning_2024}. In the sequel, we assume that such a cover, denoted by $(\mathrm{Step} \circ \cH)' \times \cG'$, has been constructed, with failure probability $\delta' = \delta/3$. All in all, $\hat{\cO}_\mathrm{sc}$ is the finite set of losses parameterized by
\begin{align*}
    (\sigma, v, f, g) \in \{ \pm 1\} \times I_{\Theta(\eps)} \times (\mathrm{Step} \circ \cH)' \times \cG'.
\end{align*}

Given that $\hat{\cO}_{\mathrm{sc}}$ is a finite set, there exist straightforward ways to use best-response oracles or no-regret learning algorithms to select a sequence of objectives. The learner's task of selecting a hypothesis $p \in \cP$ is trickier, given that $\cP$ is a rich class containing all real-valued predictors. 

\begin{proposition}[\cite{DBLP:conf/nips/HaghtalabJ023}]
    Consider $\cP$ the set of all real-valued predictors and any adversarial sequence of stochastic costs $q^{(1:T)} \in \Delta(\cO_{\mathrm{sc}})$. There exists a no-regret algorithm that outputs deterministic predictors $p^{(1:T)}$ such that $\mathrm{Reg}_{\mathrm{weak}}(p^{(1:T)}, \{ \cR_{q^{(t)}(\cdot)} \}^{(1:T)}) \leq C \sqrt{\log(2) T}$. This algorithm requires no samples from $D$. 
\label{prop:nr}
\end{proposition}

The proposed algorithm involves running Hedge at each $x \in \cX$, then aggregating by taking $p^{(t)}(x)$ to be the prediction made by the instantiation of Hedge associated with $x$ at time $t$. By the linearity of expectation, the aggregated algorithm inherits the Hedge regret bound. We defer further details to Theorems 3.7 and 5.2 in Haghtalab, Jordan, and Zhao \cite{DBLP:conf/nips/HaghtalabJ023}.

\subsection{Deterministic Step Calibration}

\begin{algorithm}[tb]
\caption{Deterministic Step Calibration}
\label{alg:detstep2}
\begin{algorithmic}[1]
\Require $\epsilon, \delta \in (0,1)$, $T \in \mathbb{N}$, $c \in [0, 1]$, $C \in \mathbb{N}$, sampling access to $D$, best response oracle $\mathcal{A}$
\State Initialize Hedge iterate $p^{(1)} \;\gets\; [\tfrac 12,\tfrac 12]^{\mathcal{X}}$
\For{$t=1,\dots,T$}
    \State Update $\ell^{(t)} \;\gets\; \mathcal{A}\bigl(p^{(t)}, \hat{\cO}_{\mathrm{sc}}^{\gamma}, c \eps\sqrt{\gamma} \bigr)$
    \State For each $x\in\mathcal{X}$, update $p^{(t+1)}(x)\;\gets\;\mathrm{Hedge}\bigl(c_x^{(1:t)}\bigr)$, where $c_x^{(t)}(\hat y)$ is defined as
    \begin{align*}
        \frac{1}{2} \cdot \left(1 + \sqrt{\frac{\gamma}{\Pr(g^{(t)}(x)=1)}} \cdot \sigma^{(t)} \cdot \hat{y} \cdot \mathbbm{1}[p^{(t)}(x) \leq v^{(t)}] \cdot f^{(t)}(x) \cdot \mathbbm{1}[g^{(t)}(x) = 1] \right)
    \end{align*}
\EndFor
\State Draw $m\gets \tfrac{C\log(T/\delta)}{\epsilon^2}$ samples $z^{(1:m)} \sim D$
\State $t^*\;\gets\;\arg\min_{t\in[T]}\sum_{(x,y)\in z^{(1:m)}}\ell^{(t)}\bigl(p^{(t)},(x,y)\bigr)$
\State \Return $p^{(t^*)}$
\end{algorithmic}
\end{algorithm}

\begin{remark}[Re-scaling and re-centering $\hat{\cO}_{\mathrm{sc}}$]
    To facilitate the use of no-regret algorithms and best-response oracles that require losses that are bounded in $[0, 1]$, we define the re-scaled and re-centered class of objectives
    \begin{align*}
        \hat{\cO}_{\mathrm{sc}}^\gamma = \left\{ \frac{1}{2} \left( 1 + \sqrt{\gamma} \cdot \ell \right) \mid \ell \in \hat{\cO}_{\mathrm{sc}} \right\},
    \end{align*}
    where $\gamma = \min_{g \in \cG} P_g = \min_{g \in \cG} \Pr_{(x, y) \sim D}(g(x)=1)$ is a known constant, by assumption.
    
    Given a set of parameters $\theta \coloneq (\sigma, v, f, g) \in \{ \pm 1 \} \times I_{\Theta(\eps)} \times (\mathrm{Step} \circ \cH)' \times \cG'$, $\ell_\theta \in \hat{\cO}_{\mathrm{sc}}^\gamma$ takes the form
    \begin{align*}
        \ell_\theta(p, (x, y)) = \frac{1}{2} \cdot \left( 1 + \sqrt{\frac{\gamma}{P_g}} \cdot \sigma \cdot (y - p(x)) \cdot \mathbbm{1}[p(x) \leq v] \cdot  f(x) \cdot \mathbbm{1}[g(x) = 1] \right) \in \left[ 0, 1 \right].
    \end{align*}
\end{remark}

The following corollary of Theorem~\ref{thm:multiobj} is immediate.

\begin{corollary}
    A $\eps\sqrt{\gamma}$-solution to the $(D, \hat{\cO}_{\mathrm{sc}}^\gamma, \cP)$-multi-objective learning problem corresponds to a $O(\eps)$-solution to the $(D, \hat{\cO}_\mathrm{sc}, \cP)$-multi-objective learning problem, and hence to a $(\cG, \cH, O(\eps))$-step calibrated predictor. 
\label{cor:multiobj2}
\end{corollary}

Though we treat $\gamma$ as a constant, we track dependence on it in the rest of our derivations to facilitate future work that treats group membership probabilities as variables. 

The following is a more complete restatement of Theorem~\ref{thm:detstep}.
\begin{theorem}[Deterministic Step Calibration]
    Fix $\varepsilon > 0$ and $\delta \in (0, 1)$. 
    Let $d_H$ be the appropriate combinatorial dimension of $H$ and $d_G$ the VC dimension of $\cG$.
    Then with probability at least $1-\delta$ and at most 
    \begin{align}
        n = O\left( \frac{1}{\eps^{3} \gamma^{3/2}} \cdot (d_H + d_G) \cdot \log \left( \frac{1}{\eps \gamma^{1/2} \delta} \right) \right)
    \end{align}
    samples from $D$, Algorithm~\ref{alg:detstep2} returns a deterministic predictor $p^*$ that is $(\cG, \cH, \eps)$-step calibrated.
\end{theorem}

\begin{proof}
    We prove this result in two steps. 
    
    By Proposition~\ref{prop:nr}, there exists a universal constant $C > 0$ such that for $T = C\eps^{-2}\gamma^{-1}$, the learner's regret is $\mathrm{Reg}_{\mathrm{weak}}(p^{(1:T)}, \{ \cR_{\ell^{(t)}}(\cdot) \}^{(1:T)}) \leq T\eps\sqrt{\gamma}/8$. Suppose that at each iteration, the adversary $\eps \sqrt{\gamma}/8$-best responds. By Proposition~\ref{prop:nrbr}, there exists a timestep $t\in [T]$ such that the predictor $p^{(t)}$ is a $\eps\sqrt{\gamma}/4$-optimal solution to the $(D, \hat{\cO}_{\mathrm{sc}}^{\gamma}, \cP)$-multi-objective learning problem. By Proposition~\ref{prop:nrbr2}, with probability $1-\delta/3$, the predictor $p^*$ returned by the algorithm is a deterministic $\eps \sqrt{\gamma}/2$-optimal solution. Then by Corollary~\ref{cor:multiobj2}, $p^*$ is $(\cG, \cH, \eps)$-step calibrated.

    By Proposition~\ref{prop:adaptive}, the claimed sample complexity is sufficient to answer $T = C\eps^{-2}\gamma^{-1}$ adaptive $\eps \sqrt{\gamma}/8$-best response queries with probability $1-\delta/3$. 
\end{proof}

\subsection{Randomized Step Calibration}

\begin{algorithm}[tb]
\caption{Randomized Step Calibration}
\label{alg:randstep2}
\begin{algorithmic}[1]
\Require $\epsilon, \delta \in (0,1)$, $T \in \mathbb{N}$, $c \in [0, 1]$, $C \in \mathbb{N}$, sampling access to $D$
\State Initialize Hedge iterates $p^{(1)} \;\gets\; [\tfrac 12,\tfrac 12]^{\mathcal{X}}$ and $q^{(1)} = \mathrm{Uniform}(\hat{\cO}_{\mathrm{sc}}^{\gamma})$.
\For{$t=1,\dots,T$}
    \State Sample objective $\ell^{(t)} \sim q^{(t)}$ and data point $(x^{(t)}, y^{(t)}) \sim D$
    \State For each $x\in\mathcal{X}$, update $p^{(t+1)}(x)\;\gets\;\mathrm{Hedge}\bigl(c_x^{(1:t)}\bigr)$, where $c_x^{(t)}(\hat y)$ is defined as
    \begin{align*}
        \frac{1}{2} \cdot \left(1 + \sqrt{\frac{\gamma}{\Pr(g^{(t)}(x)=1)}} \cdot \sigma^{(t)} \cdot \hat{y} \cdot \mathbbm{1}[p^{(t)}(x) \leq v^{(t)}] \cdot f^{(t)}(x) \cdot \mathbbm{1}[g^{(t)}(x) = 1] \right)
    \end{align*}
    Let $q^{(t+1)} = \mathrm{Hedge}(c_{\mathrm{adv}}^{(1:t)})$ where $c_{\mathrm{adv}}^{(t)} = 1-\ell_{(\cdot)}(p^{(t)}, (x^{(t)}, y^{(t)}))$
\EndFor
\State \Return $\bp^{*}$, a uniform distribution over $p^{(1)}, \dots, p^{(T)}$
\end{algorithmic}
\end{algorithm}

The following is a more complete statement of Theorem~\ref{thm:randstep}.

\begin{theorem}[Randomized Step Calibration]
    Fix $\varepsilon > 0$ and $\delta \in (0, 1)$. 
    Let $d_H$ be the appropriate combinatorial dimension of $H$ and $d_G$ the VC dimension of $\cG$.
    Then with probability at least $1-\delta$ and at most 
    \begin{align}
        n = \tilde{O}\left( \frac{1}{{\eps^2 \gamma} } \cdot (d_H + d_G) \cdot \log \left( \frac{1}{\eps \delta} \right) \right)
    \end{align}
    samples from $D$, Algorithm~\ref{alg:randstep2} returns a randomized predictor $\bp^*$ that is $(\cG, \cH,  \eps)$-step calibrated.
\end{theorem}

\begin{proof}
    By Proposition~\ref{prop:nr}, there exists a universal constant $C > 0$ such that for $T \geq C \cdot \eps^{-2} \gamma^{-1}$, the learner's regret is $\mathrm{Reg}_{\mathrm{weak}}(p^{(1:T)}, \{ \cR_{\ell^{(t)}}(\cdot) \}^{(1:T)}) \leq T\eps\sqrt{\gamma}/12$. By Proposition~\ref{prop:hedge}, there exists a universal constant $C' > 0$ such that for $T \geq C' \cdot \eps^{-2} \gamma^{-1} \cdot (d_H + d_G) \cdot \log\left(  \eps^{-1} \right)$, the adversary's regret is $\mathrm{Reg}(q^{(1:T)}, \{1- \ell_{(\cdot)}(p^{(t)}, (x^{(t)}, y^{(t)})) \}^{(1:T)}) \leq T\eps\sqrt{\gamma}/12$. By $\ell_{(\cdot)}(p^{(t)}, (x^{(t)}, y^{(t)}))$, we mean the adversary chooses the loss $\ell \in \hat{\cO}_{\mathrm{sc}}^\gamma$ to plug into the cost function of the given form. 
    
    It remains to be shown that the adversary's regret to cost functions of the form $1 - \ell_{(\cdot)}(p^{(t)}, (x^{(t)}, y^{(t)})$ closely approximates their regret to cost functions of the form $1 - \cR_{(\cdot)}(p^{(t)})$. By Proposition~\ref{prop:stoch}, there exists a universal constant $C''>0$ such that for $T \geq C'' \cdot \eps^{-2} \gamma^{-1} \cdot (d_H + d_G) \cdot \log\left( \eps^{-1} \delta^{-1}\right)$,
    \begin{align*}
        \abs{ \mathrm{Reg}(q^{(1:T)}, \{1- \ell_{(\cdot)}(p^{(t)}, (x^{(t)}, y^{(t)})) \}^{(1:T)}) - \mathrm{Reg}(q^{(1:T)}, \{1- \cL_{\ell^{(t)}}(\cdot) \}^{(1:T)}) } \leq T\eps\sqrt{\gamma}/12
    \end{align*}
    with probability $1-\delta/3$. Invoking Proposition~\ref{prop:stoch} once again, there exists a universal constant $C'''>0$ such that for $T \geq C''' \cdot \eps^{-2} \gamma^{-1} \cdot (d_H + d_G) \cdot \log\left( \eps^{-1} \delta^{-1}\right)$,
    \begin{align*}
        \abs{ \mathrm{Reg}(q^{(1:T)}, \{1- \cL_{\ell^{(t)}}(\cdot) \}^{(1:T)}) - \mathrm{Reg}(\ell^{(1:T)}, \{1- \cL_{\ell^{(t)}}(\cdot) \}^{(1:T)}) } &\leq T\eps\sqrt{\gamma}/12
    \end{align*}
    with probability $1-\delta/3$. 
    
    By the triangle inequality, the adversary's regret $\mathrm{Reg}(\ell^{(1:T)}, \{1- \cL_{\ell^{(t)}}(\cdot) \}^{(1:T)}) \leq T\eps\sqrt{\gamma}/4$. By Proposition~\ref{prop:nrnr}, the predictor $\bp^*$ is a randomized $\eps\sqrt{\gamma}/2$-solution to the $(D, \hat{\cO}_{\mathrm{sc}}^\gamma, \cP)$-multi-objective learning problem. By Corollary~\ref{cor:multiobj2}, the randomized predictor $\bp^*$ is $(\cG, \cH, O(\eps))$-step calibrated. 

    Since the algorithm requires one sample per iteration, the sample complexity is exactly $T$. 
\end{proof}

\section{Deferred Results and Proofs from Section~\ref{sec:connect}} \label{sec:connect-proofs}
In this section, we prove that deterministic panprediction implies deterministic multi-group learning, and that randomized panprediction implies randomized multi-group learning. 

\detmulti*

\begin{proof}[Proof of Proposition~\ref{prop:detmulti}]
    Using the panprediction guarantee, the loss of any competitor hypothesis $h \in \cH$ on any group $g \in \cG$ can be bounded as
    \begin{align*}
        \EEsc{(x, y) \sim D}{\ell(h(x), y)}{g(x)=1} \geq \EEsc{(x, y) \sim D}{\ell(k_\ell(p^*(x)), y)}{g(x)=1} - \eps \cdot \sqrt{P_g^{-1}}.
    \end{align*}
    It suffices to observe that $k_\ell(p^*(x)) = \mathbbm{1}[p^*(x) \geq 0.5]$.
\end{proof}

\begin{proposition}
    Fix the zero-one loss $\ell$, groups $\cG$, and binary hypothesis class $\cH$. If a randomized predictor $\bp^* \in \Delta(\cP)$ is a $(\{\ell\}, \cG, \cH, \eps)$-panpredictor, then the binary classifier $h(x) = \mathbbm{1}[\bp^*(x) \geq 0.5]$ is a randomized $(\ell, \cG, \cH, \eps)$-multi-group learner.
\end{proposition}
The proof mirrors the previous argument and is omitted. 

\section{Additional Connections}\label{sec:add-connect}
In this section, we show that multi-group learning cannot be reduced to multiaccuracy. 

For a counterexample, consider the setting where $\cX = \{ x, x' \}$ and and $\cG = \{ \cX \}$. Suppose that under distribution $D$, $\Pr(x) = \Pr(x') = 1/2$ and $\Pr(y = 0 \mid x) = \Pr(y=1 \mid x') = 1$. Now consider the hypotheses $\cH = \{h, h'\}$ where $h(x) = 0$ and $h(x') = 1$, and $h'(x) = 1$ and $h'(x') = 0$. Both hypotheses are 0-multiaccurate since
\begin{align*}
    \abs{\EEs{(x, y) \sim D}{(y - h(x))}} &= \abs{ \frac{1}{2} \cdot (1 - 1) + \frac{1}{2} \cdot (0 - 0) } = 0, \\
    \abs{\EEs{(x, y) \sim D}{(y - h'(x))}} &= \abs{ \frac{1}{2} \cdot (0 - 1) + \frac{1}{2} \cdot (1 - 0) } = 0.
\end{align*}
But $h'$ cannot be a multi-group learner for any $\eps \in (0, 1)$. 

This counterexample can be eliminated by further conditioning on the level sets of the hypotheses, which naturally motivates notions of calibration. This shows that calibration can be useful for multi-group learning, even when we do not seek omniprediction-style guarantees.

\end{document}